\documentclass[preprint,12pt]{elsarticle}
\usepackage{imakeidx}\usepackage{enumitem}
\usepackage{pdflscape}
\usepackage{mathtools}
\usepackage{nameref}
\usepackage{makecell}
\usepackage[hidelinks]{hyperref}
\usepackage[noend]{algpseudocode}
\makeindex

\usepackage{amssymb}
\usepackage{amsthm}

\usepackage[pagewise]{lineno}
\usepackage{algorithm}
\usepackage{float}
\usepackage{algpseudocode}
\usepackage{todonotes}

\usepackage{xspace}
\usepackage{comment}
\usepackage{bm}
\usepackage{amsthm}
\usepackage{todonotes}

\newcommand{\set}[1]{\ensuremath{\{#1\}}\xspace}
\newcommand{\tuple}[1]{\ensuremath{\langle #1 \rangle}\xspace}

\newcommand{\op}[1]{\mathrm{#1}}
\newcommand{\lb}{\underline}
\newcommand{\ub}{\overline}

\newcommand{\ttt}{\texttt}

\newcommand{\pluseq}{\mathrel{+}=}
\newcommand{\minuseq}{\mathrel{-}=}

\newcommand{\mC}{\mathcal{C}\xspace}

\newcommand{\mG}{\mathcal{G}\xspace}
\newcommand{\mH}{\mathcal{H}\xspace}
\newcommand{\mI}{\mathcal{I}\xspace}

\newcommand{\mS}{\mathcal{S}\xspace}
\newcommand{\mT}{\mathcal{T}\xspace}

\newcommand{\mX}{\mathcal{X}\xspace}

\newcommand{\rational}{\ensuremath{\mathbb{Q}}\xspace}

\renewcommand{\natural}{\ensuremath{\mathbb{N}}\xspace}

\usepackage{pifont}
\usepackage{amssymb}

\newcommand{\pddl}{\textsc{pddl}\xspace}
\newcommand{\pre}{\op{pre}}

\newcommand{\cond}{\op{cond}}

\newcommand{\eff}{\op{eff}}

\renewcommand{\implies}{\rightarrow}

\newcommand{\patty}{\textsc{patty}\xspace}

\newcommand{\pattern}{{\prec\xspace}}

\newcommand{\bigand}{\bigwedge}
\newcommand{\bigor}{\bigvee}

\usepackage{ stmaryrd }

\newcommand{\asseq}{:=}

\newcommand{\smt}{\textsc{smt}\xspace}

\newcommand{\pas}{\textsc{p}a\textsc{s}\xspace}
\newcommand{\pspace}{\textsc{pspace}\xspace}

\newcommand{\spp}{\textsc{spp}\xspace}

\newcommand{\ipc}{\textsc{ipc}\xspace}
\newcommand{\ite}{\textsc{ite}\xspace}

\DeclareFontFamily{U}{mathb}{\hyphenchar\font45}
\DeclareFontShape{U}{mathb}{m}{n}{
<-6> mathb5 <6-7> mathb6 <7-8> mathb7
<8-9> mathb8 <9-10> mathb9
<10-12> mathb10 <12-> mathb12
}{}
\DeclareSymbolFont{mathb}{U}{mathb}{m}{n}
\DeclareMathSymbol{\cespatternbasic}{\mathrel}{mathb}{"CE}

\newcommand{\indpi}{\pi^{\shortdownarrow}}
\newcommand{\instradi}{\textsc{InSTraDi}\xspace}

\newcommand{\arpg}{\textsc{arpg}\xspace}

\renewcommand{\iff}{\;\leftrightarrow\;}

\newcommand{\sota}{\textsc{s}o\textsc{ta}\xspace}

\newtheorem*{example*}{Example}
\newtheorem{theorem}{Theorem}
\newtheorem{corollary}{Corollary}

\usepackage{url}

\usepackage[utf8]{inputenc} 
\usepackage[T1]{fontenc}

\newcommand{\ICE}{\textsc{ice}\xspace}
\newcommand{\ICEs}{\textsc{ice}s\xspace}

\newcommand{\IC}{\textsc{ic}\xspace}
\newcommand{\ICs}{\textsc{ic}s\xspace}

\newcommand{\IE}{\textsc{ie}\xspace}
\newcommand{\IEs}{\textsc{ie}s\xspace}

\newcommand{\tstart}{\textsc{start}\xspace}
\newcommand{\tend}{\textsc{end}\xspace}
\newcommand{\talpha}{\textsc{alpha}\xspace}
\newcommand{\tomega}{\textsc{omega}\xspace}

\newcommand{\ms}{\ensuremath{\mathrm{ms}}}
\newcommand{\ieff}{\ensuremath{\mathrm{ieff}}}
\newcommand{\pieff}{\ensuremath{\mathrm{pieff}}}
\newcommand{\icond}{\ensuremath{\mathrm{icond}}}

\newcommand{\layer}{\mathrm{layer}}
\newcommand{\snap}{\textsc{s}}
\newcommand{\snapPi}{\Call{snap}{\Pi}}

\newcommand{\tamer}{\textsc{tamer}\xspace}
\newcommand{\anml}{\textsc{anml}\xspace}
\newcommand{\anmlsmt}{\textsc{anmlsmt}\xspace}
\newcommand{\optic}{\textsc{Optic}\xspace}
\newcommand{\itsat}{\textsc{itsat}\xspace}
\newcommand{\lpg}{\textsc{lpg}\xspace}
\journal{Artificial Intelligence Journal}
\usepackage{rotating}
\usepackage{pdflscape}

\begin{document}

\begin{frontmatter}



\title{Symbolic Pattern Temporal Numeric Planning with Intermediate Conditions and Effects}


\author[unige]{Matteo Cardellini\corref{cor1}}
\ead{matteo.cardellini@unige.it}
\cortext[cor1]{Corresponding author.}
\author[unige]{Enrico Giunchiglia}
\ead{enrico.giunchiglia@unige.it}

\affiliation[unige]{organization={DIBRIS, Università di Genova},
            city={Genova},
            country={Italy}}

\begin{abstract}
Recently, a Symbolic Pattern Planning (\spp) approach was proposed for numeric planning where a pattern (i.e., a finite sequence of actions) suggests a causal order between actions. 
The pattern is then encoded in a \smt formula whose models correspond to valid plans. If the suggestion by the pattern is inaccurate and no valid plan can be found, the pattern is extended until it contains the causal order of actions in a valid plan, making the approach complete.  
In this paper, we extend the \spp approach to the temporal planning with Intermediate Conditions and Effects (\ICEs) fragment, where $(i)$ actions are durative (and thus can overlap over time) and have conditions/effects which can be checked/applied at any time during an action's execution, and $(ii)$ one can specify plan's conditions/effects that must be checked/applied at specific times during the plan execution. 
Experimental results show that our \spp planner \patty $(i)$ outperforms all other planners in the literature in the majority of temporal domains without \ICEs, $(ii)$ obtains comparable results with the \sota search planner for \ICEs in literature domains with \ICEs, and $(iii)$ outperforms the same planner in a novel domain based on a real-world application.
\end{abstract}




\end{frontmatter}


\section{Introduction}
Planning is one of the oldest tales in Artificial Intelligence \cite{McCarthy1969-MCCSPP}, telling the story of autonomous agents perceiving and acting in an environment. Depending on how the environment is represented, and how agents act upon it, in the last 30 years, several planning fragments have been explored. In this paper, we focus on the \textsl{deterministic} kind of planning, where the agent's actions have a certain outcome on the environment and the agent can fully perceive the state of the environment. Initially, there was the \textsl{classical planning} fragment \cite{fikes1971strips}, which represents the environment via \textsl{propositional} (i.e., true or false) variables and the agents' actions are instantaneous and executed sequentially. Deciding plan-existence in this fragment is \pspace-complete \cite{Bylander_1991}. If we change the representation of the environment and allow for variables to assume values in the rationals, then we move to the \textsl{numeric planning} fragment, which is undecidable even in the presence of just two numeric and one propositional variable \cite{gnad2023planning}. In \textsl{temporal planning} \cite{Fox_Long_2003}, instead, actions have durations, and their execution can overlap over time. In the specification of temporal planning introduced by \citet{Fox_Long_2003}, \textsl{conditions} on the executability of a \textsl{durative action} can be imposed to hold either overall its execution, or at its start or at its end. Action's \textsl{effects} can be applied only at its start or end. To overcome this known limitation (see \cite{smith2003case}), \textsl{temporal planning with intermediate conditions and effects (\ICEs)} \cite{valentini_temporal_2020} was introduced, where $(i)$ actions' conditions/effects, can be checked/applied at any time during an action's execution, and $(ii)$ one can specify plan's conditions/effects that must be checked/applied at specific times during the plan. Temporal planning with or without \ICEs is still \pspace-complete \cite{gigante_expressive_2022} as long as we only allow for propositional variables for representing the environment. In fact, \ICEs can be "compiled away" \cite{smith2003case,fox2004investigation,gigante_expressive_2022} paying the price of increasing the plan length linearly in the number of \ICEs \cite{nebel2000compilability}.
If we add the numeric dimension, then we move to \textsl{temporal numeric planning}, which can be undecidable as in the numeric case.
Despite the complexity, several efficient planning systems for temporal numeric planning problems have been proposed, 
traditionally, either based on \textsl{Heuristic Search} or on \textsl{Planning as Satisfiability}.  
Heuristic Search \cite{Bonet_Geffner_2001}, given a set of states reachable from the initial state, selects one of these states and extends it through a sequence of actions likely to reach a goal state, keeping track of the temporal constraints between actions  during the search, and pruning states which fail to respect them (see, e.g., \cite{optic,LPG,valentini_temporal_2020}).  
In contrast, Planning as Satisfiability (\pas) \cite{DBLP:conf/ecai/KautzS92}  
$(i)$ sets a \textsl{bound} or \textsl{number of steps} $n \in \natural$, initially $n=0$;  
$(ii)$ encodes, as a logical formula, all plans with $n$ actions which are executable from the initial state and respect the temporal constraints;  
$(iii)$ checks whether any such sequence constitutes a valid plan by enforcing the \textsl{initial and goal states} on the first and last states; and  
$(iv)$ if no plan is found, repeats the process from step $(ii)$ while incrementing $n$, (see, e.g., \cite{ShinD04,RankoohG15, Rintanen15, Rintanen17}).  

\textsl{Symbolic Pattern Planning} (\spp) \cite{DBLP:conf/aaai/CardelliniGM24,Cardellini_Giunchiglia_2025} was recently introduced for numeric planning  problems where a \textsl{pattern} $\pattern$ -- i.e., a finite sequence of \textsl{snap} actions-- suggests a causal order between actions. A pattern $\pattern_I$ is computed once in the initial state: then, starting from $\pattern = \pattern_I$ if the suggestion provided by $\pattern$ is inaccurate and no valid plan can be found, $\pattern$ is \textsl{extended} by concatenating $\pattern_I$ to it, and the procedure is iterated until $\pattern$ contains the causal order of actions in a valid plan, making the procedure complete. In \cite{DBLP:conf/kr/CardelliniG25} the \spp approach was improved, showing how to symbolically search for a valid plan  by iteratively extending (adding actions to) and compressing (removing actions from) an initially computed pattern. Specifically, the idea is to concatenate $(i)$ an initial pattern that allows reaching a state $s$ satisfying a subset of the goals,  and $(ii)$ another pattern computed from $s$, which is extended until a state satisfying new additional goals can be reached. 
At the beginning, the first pattern is computed from the initial state, and the two steps are iterated until all the goals have been satisfied, at which point a valid plan is returned. \citet{cardellini2025aij} considered several pattern selection procedures and mechanisms for improving the quality of the returned solution, performing ablation studies on the different strategies.

In this paper, we extend the \spp approach to Temporal Planning with Intermediate Conditions and Effects (\ICEs) where the pattern suggests the causal order between the intermediate conditions and intermediate effects, both relative to actions and to the plan. We perform an experimental analysis comparing our planner \patty with other temporal planners in the literature on temporal domains both with and without \ICEs. We show that our planner \patty achieves overall better performances, solving $168$ out of $190$ instances, compared to the best literature search-based solver \tamer \cite{valentini_temporal_2020}, solving $131$ instances, and the best literature \pas-based solver \anmlsmt \cite{Panjkovic_Micheli_2023}, solving $98$.

As shown by \citet{DBLP:conf/aaai/CardelliniGM24}, the pattern encoding generalizes and improves the rolled-up encoding \cite{Scala_Ramirez_Haslum_Thiebaux_2016_Rolling} and the relaxed-relaxed-$\exists$ encoding \cite{balyo_relaxing_2013, bofill_espasa_villaret_2016}, both used in the satisfiability approach. In search-based planning, planners for classical planning like \textsc{yahsp} \cite{vidal2004yahsp} exploit sequences of actions (look-ahead plans) to reach intermediate states, similar to macro-actions \cite{alarnaouti2024macroactions}. Patterns, however, generalize these ideas by allowing subsequences of arbitrary actions and accounting for repeated applications of actions (rolls).  On the modelling side, \cite{DBLP:conf/ijcai/BonassiGS22} introduces \textsl{planning with action constraints} (\textsc{pac}) in \pddl{3}, specifying expected action trajectories to guide planning. While unrelated to patterns, it illustrates that guiding the order of actions can be beneficial in many domains. The \tamer planner \cite{valentini_temporal_2020}, built specifically for temporal planning with \ICEs, implements a planning as search strategy, where each state is enriched with a \textsl{Simple Temporal Network} (\textsc{stn}) \cite{dechter1991temporal} keeping track of the temporal order between actions, pruning states as soon as the \textsc{stn} becomes infeasible. This paper builds on the conference paper \cite{Cardellini_Giunchiglia_2025}, where the \spp approach for temporal numeric planning (without \ICEs) was initially presented.

This paper is structured as follows. After the presentation of a motivating example in Sec.\ \ref{sec:motivating}, we formalize a Temporal Numeric Planning with \ICEs task in Section \ref{sec:ices-formal}. Then, in Sec.\ \ref{sec:approach} we formalize our approach: $(i)$ in Sec.\ \ref{sec:rolling} we extend the concept of \textsl{rolling}, taken from \cite{Scala_Ramirez_Haslum_Thiebaux_2016_Rolling}, to durative actions, $(ii)$ in Sec.\ \ref{sec:pattern} we discuss the pattern, how it is defined (Sec.\ \ref{sec:pattern-definition}) and how it is computed (Sec.\ \ref{sec:pattern-computation}), $(iii)$ in Sec.\ \ref{sec:spp} we propose our \spp encoding to deal with both the causal and temporal aspects of temporal planning with \ICEs, and $(iv)$ in Sec.\ \ref{sec:correctness-completeness} we prove its correctness and completeness. Finally, in Sec.\ \ref{sec:experimental} we perform an experimental analysis on our motivating example and several domains from the literature.

\subsection{Motivating Example} \label{sec:motivating}

Throughout this paper, we employ a motivating example to illustrate, ground and benchmark our proposed approach. 
The In-Station Train Dispatching (\instradi) Problem \cite{lamorgese2015exact,cardellini_-station_2021} consists in planning (or \textsl{dispatching}) the movements of trains inside a railway station. 
An example of a railway station is presented in Fig. \ref{fig:instradi}. 
The smallest component of a station is the \textsl{track circuit}\footnote{The name ``circuit" comes from the fact that each rail is connected to a pole of a battery and a relay. When a train moves on top of a track circuit, its axles close the circuit and make current flow through the relay, signalling the presence of a train.}, i.e., a segment of a \textsl{track} (i.e., two \textsl{rails}). 
\textsl{Signals}\footnote{Signals can be of many types, e.g., electrical and/or mechanical. In our simplified version, we will use only (electrical) traffic lights.} authorize trains' movements in the station. Usually, signals are in proximity of \textsl{entry points} of the station (e.g., signals \ttt{03} and \ttt{01} in Fig. \ref{fig:instradi}), \textsl{exit points} (e.g., \ttt{02} and \ttt{04}) and \textsl{platforms} (e.g., \ttt{21}, \ttt{41}, \ttt{22}, \ttt{42}, \ttt{23} and \ttt{43}). 
Entry and exit points delimit the station from the external line, usually connecting directly to another railway station. Trains can stop in platforms, allowing passengers to board/alight. Formally, a platform is a set of track circuits (e.g., platform \ttt{I} is composed of track circuits \ttt{104}, \ttt{105} and \ttt{106}). 
To simplify the movements of trains, human dispatchers use \textsl{routes}, i.e., sequences of track circuits, to move a train between two signals. Fig. \ref{fig:instradi} shows two routes in blue and in red. For example, route \ttt{03-23} (in blue) goes from signal \ttt{03} to signal \ttt{23} and is composed of track circuits\footnote{In our example, we skip, for simplicity, the concept of \textsl{switches}, that allow trains to move between parallel track circuits, like the one between \ttt{108} and \ttt{118}.} \ttt{109}, \ttt{108}, \ttt{118}, \ttt{124}, \ttt{123}. Before moving through a route, to guarantee safety, the route must be \textsl{reserved}, achievable only if all the track circuits of the route are not occupied by another train.

\begin{figure}
    \centering
    \includegraphics[width=1\linewidth]{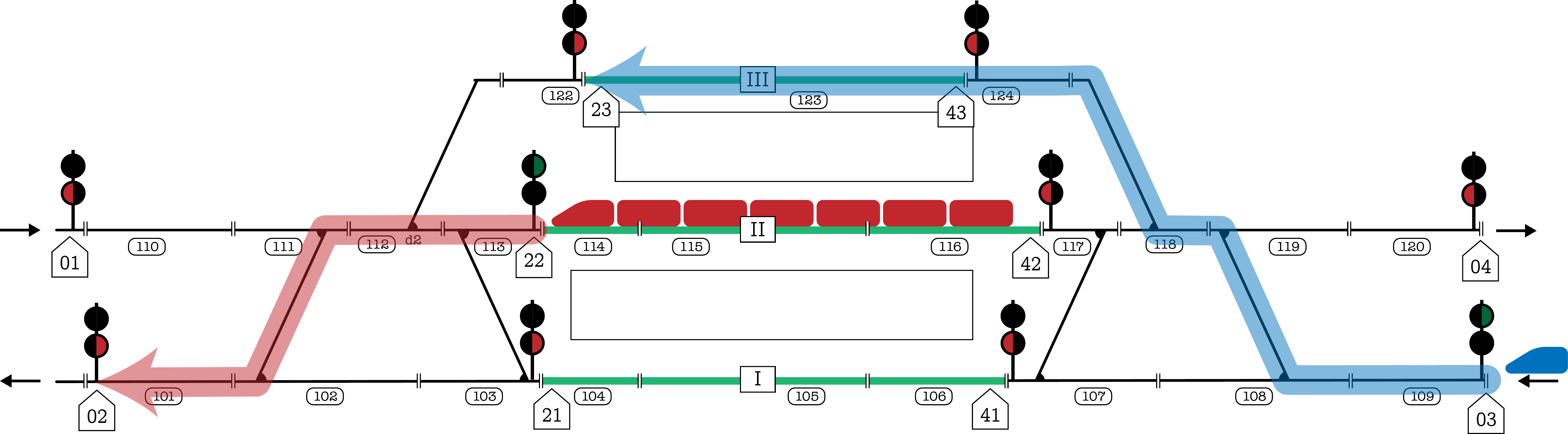}
    \caption{The motivating \instradi example with two trains, \ttt{red} and \ttt{blue}. The train \ttt{red} is stopping at platform $\ttt{II}$ oriented towards the green signal $\ttt{22}$ which allows entering route \ttt{22-02} to exit the station from the exit point \ttt{02}. Train \ttt{blue} is about to enter the station from the entry point \ttt{03} using route \ttt{03-23} to reach platform \ttt{III}.}
    \label{fig:instradi}
\end{figure}

In this work, we suppose to be provided with a \textsl{forecast system} to estimate, for each train, the travel times on track circuits and routes, the stopping time at platforms and the estimated arrival of trains at entry points\footnote{The simulation of the actual timings of the movements of trains inside the station would require knowing the length of track circuits, platforms and trains, as well as how trains accelerate and decelerate inside the station. To avoid this burden, machine learning approaches (see e.g., \cite{boleto_-station_2021}) have been implemented to provide forecast times with high accuracy, straightforwardly employable in our model.}. A \textsl{nominal timetable}, determines trains' earliest depart time from a platform and their destination (i.e., the exit point they must exit). While each train can arrive at a platform earlier or later than what the nominal timetable determines, it cannot surely depart earlier. 
Further, when exiting the station, trains have to respect a \textsl{priority}, i.e., an exit order between the trains, used to avoid that faster trains get stuck behind slower ones while on the tracks connecting to another station. 
Finally, during the operation, it is possible that moving through some track circuits in a window of time is forbidden, due to \textsl{maintenance}.

A \textsl{dispatchment}, i.e., a solution to the \instradi problem, is a set of \textsl{timed commands}, denoting at which time each train is authorized to enter a route. A dispatchment is valid when
\begin{enumerate}
    \item the safety is guaranteed, i.e., at any time, at most one train is on a track circuit,
    \item trains respect the nominal timetable, i.e., they $(i)$ stop at a platform, $(ii)$ haven't left the platform earlier than announced, and $(iii)$ have exited the station from the right exit point,
    \item the priority between trains exiting the station is respected,
    \item no train moves through the maintained track circuits during the maintenance windows.
\end{enumerate}

\section{Temporal Numeric Planning with ICEs} \label{sec:ices-formal}

In this section, we define Temporal Numeric Planning with \ICEs tasks, and then we exemplify how the motivating example can be modelled in the introduced language. 

\subsection{The Planning Fragment}

\paragraph{Planning Task} 
A \textsl{temporal numeric planning task with Intermediate Conditions and Effects\footnote{This formalization is partially inspired by \citet{valentini_temporal_2020}, keeping the general idea but extending it with numeric variables.} (\ICEs)} (in the following just "{\sl planning task}") is a tuple $\Pi = \tuple{X, A, I,G,C,E}$, where $X = \tuple{V_B,V_N}$ is a pair of finite sets containing \textsl{propositional} ($V_B$) and \textsl{numeric} ($V_N$) variables, with domains $\set{\top, \bot}$ (i.e., truth and falsity), and rationals in $\rational$, respectively. In the following, we use $v \in X$, to denote any propositional or numeric variable in $V_B \cup V_N$.
A \textsl{propositional condition} for a propositional variable $v$ is either $v = \top$ or $v=\bot$. A \textsl{numeric condition} has the form $\psi\unrhd 0$, where\footnote{In examples, for brevity, we will also use $\psi \leq 0$ instead of $-\psi \geq 0$, $\psi < 0$ instead of $-\psi > 0$ and $\psi = 0$ instead of the two conditions $\psi \geq 0$ and $-\psi \geq 0$.} $\unrhd \in \{\ge, >\}$ and $\psi$ is a \textsl{linear expression} over $V_N$, i.e., an expression of the form $\sum_{x\in V_N} k_x x + k$, with $k_x, k
\in \rational$. If, in a linear expression $\psi$, we have the coefficient $k_x \neq 0$, we say that $\psi$ \textsl{contains}~$x$. A \textsl{propositional effect} is of the form $v \asseq \top$ or $v \asseq \bot$ with $v \in V_B$. A \textsl{numeric effect} is of the form $x \asseq \psi$, with  $x \in V_N$ and $\psi$ a linear expression. We shorten a \textsl{linear increment}, i.e., a numeric effect $x := x + \psi'$, with $\psi'$ a linear expression not containing $x$, as $x \pluseq \psi'$.
In the planning task $\Pi$, the set $A$ is a finite set of \textsl{durative actions}. A durative action $b$ is a tuple $\tuple{\icond(b), \ieff(b), [L,U]}$, where $\icond(b)$ and $\ieff(b)$ are sets of \textsl{action intermediate conditions} and \textsl{action intermediate effects} of $b$, while $L, U \in \mathbb{Q}^{\geq 0}$, with $L \le U$, are the {\sl bounds} on the  duration of $b$. An intermediate condition (\IC) is a tuple $c = \tuple{\tau^\vdash, \tau^\dashv, \cond(c)}$ where $\tau^\vdash$ and $\tau^\dashv$ are the \textsl{relative times} denoting the start and end of when $c$ has to hold, and $\cond(c)$ is a set of propositional or numeric conditions. An intermediate effect (\IE) is a tuple $e = \tuple{\tau, \eff(e)}$ where $\tau$ is a relative time denoting when $e$ is applied and $\eff(e)$ is a set of propositional and numeric effects. We assume that for each variable $v \in X$, $v$ occurs in $\op{eff}(e)$ at most once to the left of the operator ``$\asseq$'', and when this happens we say that $v$ is \textsl{assigned} by $e$.
A \textsl{relative time} is either \textsl{action-relative} or \textsl{plan-relative}. An action-relative time is of the form $\tstart + k$ or $\tend -k$ where $(i)$ $\tstart$ and $\tend$ are \textsl{anchors} for the starting or ending time, respectively, of the durative action containing the \IC/\IE, and $(ii)$ $k\in \rational^{\ge 0}$ is the \textsl{offset} from the anchor. A plan-relative time is either $\talpha + k$ or $\tomega - k$, with $k \in \rational^{>0}$, where $\talpha$ is the plan's first executed action start time (usually $0$) and $\tomega$ is the plan's last action end time\footnote{Notice that for a plan-relative time, we assume $k > 0$, since it doesn't make sense to check a condition or impose an effect in the initial or final state.}. Depending on whether the relative times are action-relative or plan-relative we categorize the \ICEs into \textsl{action-\ICs, plan-\ICs, action-\IEs} and \textsl{plan-\IEs}, respectively. Let $b = \tuple{\icond(b), \ieff(b), [L,U]}$ be a durative action, it is clear that $L$ can be assumed to be not lower than all the offsets $k$ in the relative times $\tstart + k$ and $\tend - k$ in $\icond(b)$ and $\ieff(b)$, since each  condition/effect of $b$ has to take place while the action is in execution. 
Moreover, to make notation easier along the paper, we assume to always have \ICs both at the start and at the end of $b$, e.g., adding in $\icond(b)$ two always-respected \ICs $\tuple{\tstart, \tstart, \set{0=0}}$ and $\tuple{\tend, \tend, \set{0=0}}$.
A \textsl{snap} or \textsl{instantaneous action} is a durative action with $L=U=0$ and where all the \ICEs in $\icond(b) \cup \ieff(b)$ have $\tstart$ as anchor in the relative time. To simplify notation, we refer to snap actions as a pair $a = \tuple{\pre(a), \eff(a)}$ instead of the longer form for a durative action $\tuple{\set{\tuple{\tstart, \tstart, \pre(a)}}, \set{\tuple{\tstart, \eff(a)}}, [0,0]}$. For a durative action $b = \tuple{\icond(b), \ieff(b), [L,U]}$ we assume that if $U > 0$, then $L > 0$, i.e., an action can be instantaneous only if specified as $L=U=0$. In the following, we will use the symbols $v,w,x$ to denote variables, $\psi$ to denote expressions, $b$ to denote durative actions, and $a$ to denote snap actions, possibly decorated with super/subscripts.
Finally, in the planning task $\Pi$: $I$ is the \textsl{initial state}, where a \textsl{state} is a total assignment of the variables in $X$ to their respective domains; $G$ is a set of conditions called \textsl{goals} or {\sl goal  conditions}; $C$ and $E$ are sets of plan-\ICs and plan-\IEs, respectively.
We assume that both $(i)$ in $\icond(b)$ for each $b \in A$ and $(ii)$ in $C$, there are no distinct pair of \ICs $c_1 = \tuple{\tau^\vdash_1, \tau^\dashv_1, \cond(c_1)}$ and $c_2 = \tuple{\tau^\vdash_2, \tau^\dashv_2, \cond(c_2)}$ with $\tau^\vdash_1 = \tau^\vdash_2$ and $\tau^\dashv_1 = \tau^\dashv_2$ (analogously for \IEs in $\ieff(b)$ for each $b \in A$ and in $E$). If this was the case, we could join them in only one \IC $\tuple{\tau^\vdash_1, \tau^\dashv_1, \cond(c_1) \cup \cond(c_2)}$.

\paragraph{Plan and Plan Validity} Consider a planning task 
$\Pi = \tuple{X, A, I,G,C,E}$.
We denote with $\icond(\Pi)$ and $\ieff(\Pi)$ the sets of all plans/action-\ICs and plan/action-\IEs of $\Pi$, i.e.,
\begin{flalign}
    \icond(\Pi) &= C \cup \bigcup_{b \in A} \icond(b), \label{eq:icondPi} \\
    \ieff(\Pi) &= E \cup \bigcup_{b \in A} \ieff(b). \label{eq:ieffPi}
\end{flalign}
We say that two distinct \IEs $e$ and $e'$ 
\begin{enumerate}
    \item \textsl{interfere} if either both $e$ and $e'$ assign $v \in X$, or $x' := \psi' \in \eff(e')$ and $x := \psi \in \eff(e)$ with $\psi'$ containing $x$, and $x,x' \in V_N$, and
    \item are in \textsl{mutex} if $e$ interferes with $e'$ or $e'$ interferes with $e$.  
\end{enumerate} 
Let $e$ be an \IE and $c$ be an \IC. We say that $e$ and $c$ are in {\sl mutex} if $e$ assigns a variable $v$ and either $v = \top$, $v = \bot$, or $\psi \unrhd 0 $ with $\psi$ containing $v$, are in $ \cond(c)$.

A \textsl{timed durative action} is a tuple $\tuple{t, b, d}$ with  $t \in \rational^{> 0}$ being the \textsl{absolute time} in which the durative action $b = \tuple{\icond(b), \ieff(b), [L, U]}$ is executed, lasting a duration $d \in [L,U]$. A \textsl{temporal (numeric) plan} $\pi$ for $\Pi$ is a finite set of timed durative actions. The \textsl{make-span} of $\pi$ is $\ms(\pi) = \max(\set{t+d \mid \tuple{t,b,d} \in \pi})$. The \textsl{absolute \ICs} and \textsl{absolute \IEs of $\pi$}, denoted with $\icond(\pi)$ and $\ieff(\pi)$ respectively, are the sets
{\small
\begin{flalign}
    \icond(\pi) =\;& \set{\tuple{\tau^\vdash[t, t+d], \tau^\dashv[t, t+d], c} \mid \tuple{t,b,d} \in \pi, \tuple{\tau^\vdash, \tau^\dashv, c} \in \icond(b)} \nonumber\\
    \cup \; & \set{\tuple{\tau^\vdash[0, \ms(\pi)], \tau^\dashv[0, \ms(\pi)], c} \mid \tuple{\tau^\vdash, \tau^\dashv, c} \in C}, \label{eq:icond}\\
    \ieff(\pi) =\;& \set{\tuple{\tau[t,t+d], e} \mid \tuple{t,b,d} \in \pi, \tuple{\tau, e} \in \ieff(b)}\nonumber\\
    \cup\;& \set{\tuple{\tau[0, \ms(\pi)], e} \mid \tuple{\tau, e} \in E}, \label{eq:ieff}\\
    \text{ where } \tau[a, b] =\; & \begin{cases}
        a + k & \;\;\text{if } \tau = \tstart + k \text{ or } \tau = \talpha + k,\\
        b - k & \;\;\text{if } \tau = \tend - k \text{ or } \tau = \tomega - k.
    \end{cases}\label{eq:tau_a_b}
\end{flalign}
}
Intuitively, the sets are called "absolute" since \ICEs have the time expressed as absolute (i.e., in $\rational^{> 0}$) instead of relative (e.g., $\tstart + k$). In the following $\tau$ denotes a relative time and $t$ an absolute one.

Let $S$ be a state. For each linear expression $\psi$, let  $S(\psi)$ be the evaluation of the expression obtained by substituting each variable $x \in V_N$ in $\psi$ with $S(x)$. Let $\Gamma$ be a set of conditions. We say that a state $S$ {\sl satisfies} the conditions in $\Gamma$, denoted with $S \models \Gamma$ if for each $v = \top$, $w = \bot$, $\psi \unrhd 0 \in \Gamma$ we have $S(v) = \top$, $S(w) = \bot$, $S(\psi) \unrhd 0$, respectively. Let $e = \tuple{t, \eff(e)}$ be an absolute \IE. The {\sl state resulting from applying $e$ in $S$} is the state $S' = res(S,e)$ such that for each $v \in X$ we have $S'(v) = \top$ if $v := \top \in \eff(e)$,  $S'(v) = \bot$ if $v := \bot \in \eff(e)$,  $S'(v) = S(\psi)$ if $v := \psi \in \eff(e)$, and  $S'(v) = S(v)$ otherwise. 
Let $\pi$ be a plan, and let $\ieff(\pi)$ $ = \set{\tuple{t_1, \eff(e_1)}, \dots, \tuple{t_n, \eff(e_n)}}$ be the absolute effects of $\pi$ (Eq. \ref{eq:ieff}). The \textsl{parallel absolute effects of $\pi$}, is a sequence $\pieff(\pi) = \tuple{t_1, \eff(e_1)};\dots;\tuple{t_m, \eff(e_m)}$ where each $\eff(e_i)$ is obtained by joining the effects of all the \IEs in $\ieff(\pi)$ that are applied at the same time $t_i$. Clearly, $t_1 < \dots < t_m$ and $m \leq n$. Let $S_0$ be a state. The sequence $\pieff(\pi)$ applied from $S_0$ induces a sequence of \textsl{timed states} $S_0(\pi) = \tuple{t_0=0,S_0};\tuple{t_1,S_1}\dots;\tuple{t_m,S_m}$ where  $S_i = res(S_{i-1}, e_i)$ for $i \in [1,m]$. We say that $\pi$ is {\sl valid} iff 
\begin{enumerate}
    \item \textsl{initial condition}: $S_0$ is the initial state $I$.
    \item \textsl{goal condition}: $S_m \models G$,
    \item \textsl{intermediate conditions:} for each $c = \tuple{t^\vdash, t^\dashv, \cond(c)} \in \icond(\pi)$ and 
    \begin{enumerate}
        \item for $\tuple{t_0, S_0}$ starting $S_0(\pi)$, if $t^\vdash = t_0$, then $S_0 \models \cond(c)$,
        \item for each pair $\tuple{t_i, S_i},\tuple{t_{i+1}, S_{i+1}}$ in $S_0(\pi)$ such that either $t_i < t^\vdash \leq t_{i+1}$ or $t_i < t^\dashv \leq t_{i+1}$ it holds that $S_i \models \cond(c)$, 
        \item for $\tuple{t_m, S_m}$ ending $S_0(\pi)$, if $t_m < t^\dashv$, then $S_m \models \cond(c)$.
    \end{enumerate}
    Intuitively, if a condition holds in $t^\vdash = t_{i+1} = t^\dashv$ and an effect is applied in $t_{i+1}$, the condition has to be satisfied by the state before the application of the effect, i.e., $S_i$. Similarly, if a condition finishes to hold in $t^\dashv = t_{i+1}$, then it must be satisfied by $S_i$. If a condition holds from $t^\vdash$ to $t^\dashv$ then for each $\tuple{t_i, S_i} \in S_0(\pi)$  with $t^\vdash < t_i < t^\dashv$ it must hold on $S_i$. If a condition must hold in between two states $S_i$ and $S_{i+1}$, i.e., $t_i < t^\vdash \leq t^\dashv < t_{i+1}$ then the condition must be satisfied by $S_i$. Finally, if a condition must be satisfied at any time after $t_m$, then it must be satisfied by $S_m$.
    \item \textsl{no self-overlapping:} for each  $\tuple{t,b,d}, \tuple{t',b,d'}$ in $\pi$, i.e., the application of the same durative action $b$ at different times and with (possibly) different durations, it holds that either $t' +d' \leq t$ or $t' \geq t + d$,
    \item \textsl{$\epsilon$-separation:} let $\epsilon \in \rational^{\geq 0}$ and 
    let $e = \tuple{t, \eff(e)}, e' = \tuple{t', \eff(e')} \in \ieff(\pi)$ be two distinct \IEs in $\pi$. If $e$ and $e'$ are in mutex then  $|t - t'| \geq \epsilon$.
\end{enumerate}
We thus considered the standard notion of validity used, e.g., in \cite{Fox_Long_2003,RankoohG15,Haslum_pddl_2019,Panjkovic_Micheli_2024, Cardellini_Giunchiglia_2025}. In the case where $\Pi$ has no numeric variables, the no self-overlapping property ensures that the problem of deciding plan-existence for $\Pi$ is $\pspace$-complete, while allowing self-overlapping makes the problem $\textsc{expspace}$-complete or undecidable, depending on whether we impose $\epsilon$-separation or not \cite{gigante_decidability_2022}. In the case with numeric variables, the problem is undecidable even with $\epsilon$-separation \cite{Helmert_2002, gnad2023planning}. 

\subsection{Motivating Example in Temporal Numeric Planning with \ICEs}

We now model our \instradi motivating example presented in Section~\ref{sec:motivating} using the Temporal Numeric Planning with \ICEs formalism. 

Let's suppose we are in the setting presented in Fig. \ref{fig:initial-state}, in which there are two trains,  called \ttt{red} and \ttt{blue}. The \ttt{red} train $(i)$ is occupying platform $\ttt{II}$ and track circuits \ttt{114}, \ttt{115} and \ttt{116}, facing signal $\ttt{22}$, $(ii)$ will depart, according to the nominal timetable, in at least $30s$, and $(iii)$ must exit from \ttt{02}. The \ttt{blue} train $(i)$ is arriving in signal \ttt{03} in $5s$, $(ii)$ it must stop at some platform and departure in at least $100$s, and $(iii)$ exit from \ttt{02}. The \ttt{red} train must exit before the \ttt{blue} train. Suppose that in the interval in between the next $30s$ to $50s$ the track circuit $\ttt{102}$ is under maintenance. 

We can model such  \instradi problem with the planning task $\Pi = \tuple{X,A,I,G,C,E}$, where:
\begin{figure}
    \centering
    \includegraphics[width=\linewidth]{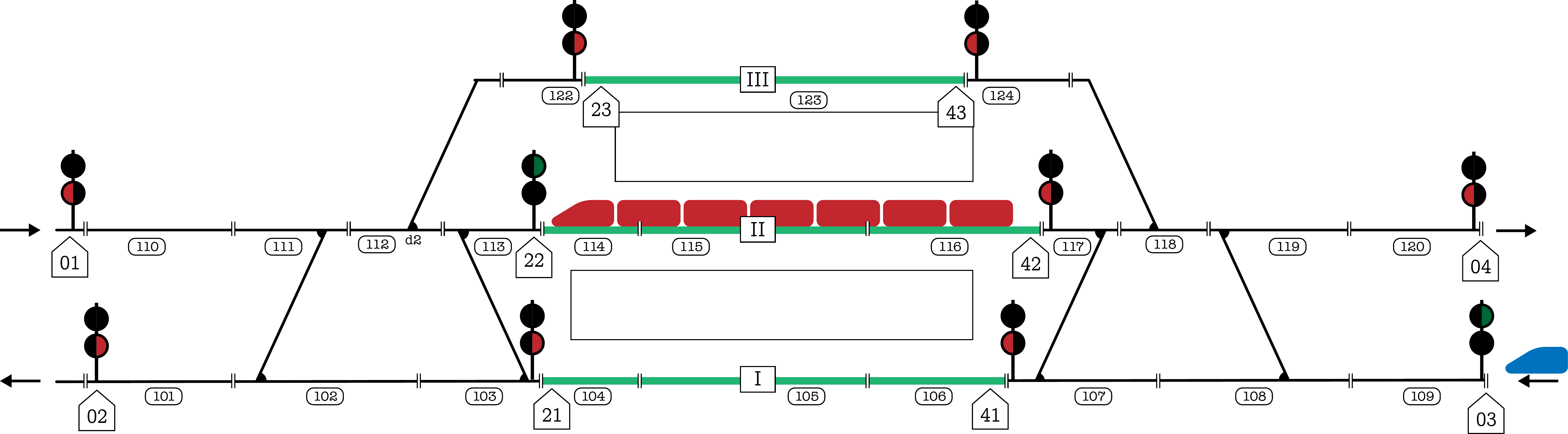}
    \caption{Initial state of our motivating example. The \ttt{red} train is occupying platform $\ttt{II}$. The \ttt{blue} train is arriving in $5s$ in \ttt{03}.}
    \label{fig:initial-state}
\end{figure}
\begin{description}
    \item The propositional variables in $X$ are\footnote{We present the variables in a “lifted” format to simplify the notation, but they have to be grounded for each train, route, signal and track circuit.}
    \begin{description}[labelsep=1pt]
        \item $occupied(c)$: track circuit $c$ is occupied by a train,
        \item $moved(\tau, r)$: train\footnote{In this paper, $\tau$ represents both a relative time and a train. The meaning will be clear from the context.} $\tau$ has moved through route $r$,
        \item $inFront(\tau, s)$: train $\tau$ is waiting in front of signal $s$,
        \item $green(\tau, s)$: signal $s$ is green for train $\tau$,
        \item $stopping(\tau, p)$: train $\tau$ is stopping on platform $p$,
        \item $stopped(\tau)$: train $\tau$ has stopped at a platform,
        \item $stoppedAt(\tau, p)$: train $\tau$ has stopped at platform $p$,
        \item $timetable(\tau)$: the departure time of train $\tau$ has passed,
        \item $\mathit{left}(\tau, s)$: train $\tau$ has left the station from signal $s$.
    \end{description}
    \item The numeric variables in $X$ are
    \begin{description}
        \item $\mathit{exitCounter}(s)$: counting the trains which have left from signal $s$,
        \item $order(\tau)$: denoting the priority of train $\tau$ as an integer.
    \end{description}
    \item The initial state $I$ has the following propositional variables set to true: 
    \begin{gather*}
        stoppedAt(\ttt{red}, \ttt{II}), 
        stopped(\ttt{red}), inFront(\ttt{red}, \ttt{22}),\\
        occupied(\ttt{114}), occupied(\ttt{115}), occupied(\ttt{116}).
    \end{gather*}
    All other propositional variables are set to false. For the numeric variables we have in $I$ that
    $I(\mathit{exitCounter}(\ttt{02})) = 0$, $I(\mathit{exitCounter}(\ttt{04})) = 0$, $I(order(\ttt{red})) = 0$ and $I(order(\ttt{blue})) = 1$, imposing that \ttt{red} must exit before \ttt{blue}.
    \item The goal $G$ is the set of conditions\footnote{To simplify the notation, we indicate $v = \top$ and $v := \top$ with $v$ and $w = \bot$ and $w := \bot$ with $\neg w$. The difference between conditions and effects will be clear from the context.}
    \begin{flalign*}
        \{&stopped(\ttt{red}), stopped(\ttt{blue}), \mathit{left}(\ttt{red} , \ttt{02}), \mathit{left}(\ttt{blue}, \ttt{02})\}.
    \end{flalign*}
    \item The set $C$ of plan-\ICs state that, in the maintenance period, the track circuit \ttt{102} must not be occupied, i.e.,
    $$\set{\tuple{\talpha + 30, \talpha + 50, \set{\neg occupied(\ttt{102})}}}.$$
    \item The set $E$ of plan-\IEs model $(i)$ the arrival of train \ttt{blue} in $5s$, $(ii)$ the departure time in $60$s and $100$s in the nominal timetable of \ttt{red} and \ttt{blue}, respectively
    \begin{flalign*}
        \{\tuple{\talpha + 5,& \set{inFront(\ttt{blue}, \ttt{03}), green(\ttt{blue}, \ttt{03})}},\\
        \tuple{\talpha + 30,& \set{timetable(\ttt{red})}},\\
        \tuple{\talpha + 100,& \set{timetable(\ttt{blue})}}\}
    \end{flalign*}
    \item The set of actions $A$ contains the following actions:
    \begin{description}
        \item $move(\tau, r)$ models the movement of train $\tau$ in route $r$. Let's take for example the action $b = move(\ttt{blue}, \ttt{03-23})$ to perform the movement of the blue arrow in Fig. \ref{fig:instradi}. Let's suppose that the forecast predicts that the blue train takes $30s$  to run through route \ttt{03-23}, and each track circuit in the route takes $5s$. Thus, the action $b = move(\ttt{blue}, \ttt{03-23})$ is modelled as $\tuple{\icond(b), \ieff(b), [30,30]}$ where
        \begin{equation*}
        \begin{array}{@{}r@{}r@{}l@{}l@{}l}
            \icond(b) = &\set{&\tuple{\tstart, \tstart, &\set{&inFront(\ttt{blue}, \ttt{03}),green(\ttt{blue}, 03), \\
            &&&&\neg  occupied(\ttt{109}), \neg occupied(\ttt{108}), \\
            &&&& \cdots, \\
            &&&& \neg occupied(\ttt{124}),\neg occupied(\ttt{123})}}}\\
            \eff(b) = &\set{&\tuple{\tstart, &\set{&\neg green(\ttt{blue}, 03), \neg inFront(\ttt{blue}, \ttt{03}),\\
            &&&& moved(\ttt{blue}, \ttt{03-23}), \\
            &&&& occupied(\ttt{109}), occupied(\ttt{108}), \\
            &&&& \cdots, \\
            &&&& occupied(\ttt{124}), occupied(\ttt{123})}},\\
            &&\tuple{\tstart + 5, &\set{&\neg occupied(\ttt{109})}},\\
            &&\tuple{\tstart + 10, &\set{&\neg occupied(\ttt{108})}},\\
            && \cdots \\
            &&\tuple{\tstart + 25, &\set{&\neg occupied(\ttt{124})}},\\
            &&\tuple{\tend, &\set{&inFront(\ttt{blue}, \ttt{23})}}
            }.
        \end{array}
    \end{equation*}
    Intuitively, before applying the action, the \ttt{blue} train must be in front of the green signal \ttt{03} and no track circuit of route \ttt{03-23} must be occupied by any other train. When the movement starts, the \ttt{03} signal is put to red, and we set all the track circuits of \ttt{03-23} as occupied, to guarantee safety. While the train moves, the \IEs of the action free the track segments (except the last one). When the action has ended and the movement completed, the train is now in front of signal \ttt{23}.
    \item $stop(\tau, s, p)$ models the stop of train $\tau$ facing signal $s$ of platform $p$. Let's take action $b = stop(\ttt{blue}, \ttt{23}, \ttt{III})$ and suppose that it is estimated that the blue train takes $35s$ to board/light passengers. Thus, the action is modelled as $\tuple{\icond(b), \ieff(b), [35,35]}$ where
    \begin{equation*}
        \begin{array}{@{}r@{}r@{}l@{}l@{}l}
            \icond(b) = &\set{&\tuple{\tstart, \tstart, &\set{&inFront(\ttt{blue}, \ttt{23}), \\
            &&&&\neg  stopped(\ttt{blue}), \neg stopping(\ttt{blue}, \ttt{II})}}}\\
            \eff(b) = &\set{&\tuple{\tstart, &\set{&stopping(\ttt{blue}, \ttt{II}}}\\
            &&\tuple{\tend, &\set{&stopped(\ttt{blue}), \neg stopping(\ttt{blue}, \ttt{II}),\\
            &&&&stoppedAt(\ttt{blue}, \ttt{III})}}
            }.
        \end{array}
    \end{equation*}
    \item $depart(\tau, p, r)$ models the departure of train $\tau$ from platform $p$ into route $r$. Let's take the action $b = depart(\ttt{blue}, \ttt{III}, \ttt{23-02})$. The action lasts as long as the time to free the platform's track circuits, let us suppose $10s$, and so is modelled as $\tuple{\icond(b), \ieff(b), [10,10]}$ where
    \begin{equation*}
        \begin{array}{@{}r@{}r@{}l@{}l@{}l}
            \icond(b) = &\set{&\tuple{\tstart, \tstart, &\set{&inFront(\ttt{blue}, \ttt{23}), \\
            &&&&stoppedAt(\ttt{blue}, \ttt{II}), \\
            &&&&timetable(\ttt{blue})}},\\
             &&\tuple{\tstart + 2\epsilon,\tstart + 2\epsilon, &\set{& moved(\ttt{blue}, \ttt{23-02})}}},\\
            \eff(b) = &\set{&\tuple{\tstart, &\set{&green(\ttt{blue}, \ttt{23})}}\\
            &&\tuple{\tend, &\set{&\neg occupied(\ttt{123})}}
            }.
        \end{array}
    \end{equation*}
    Intuitively, the train can depart only if it is in front of the stop signal, has already stopped, and the nominal timetable time has already passed. The departure action sets the signal to green and frees the platform's track circuits upon completion. The value $\epsilon$ in the \IC is the same $\epsilon$ chosen for the $\epsilon$-separation condition for plan validity in Section \ref{sec:ices-formal}. Its purpose is to force the $move$ action to be executed immediately after the $depart$ action starts: the only valid plan indeed is one where $depart(\ttt{blue}, \ttt{III}, \ttt{23-02})$ is executed at some time $t$ and the $move(\ttt{blue}, \ttt{23-02})$ action is executed immediately after, at $t + \epsilon$, so that at $t + 2\epsilon$ the condition $\set{moved(\ttt{blue}, \ttt{23-02})}$ is respected.
    \item $exit(\tau, r)$ models the exit from the station of train $\tau$ which has finished route $r$. The action is instantaneous and is only grounded for exit points, e.g., $a = exit(\ttt{blue}, \ttt{23-02})$. It is thus modelled as $\tuple{\pre(a), \eff(a)}$ where
    \begin{flalign*}
            \pre(a) = \set{&inFront(\ttt{blue}, \ttt{02}),
             exitCounter(\ttt{02}) = order(\ttt{blue})},\\
            \eff(a) = \set{&\mathit{left}(\ttt{blue}, \ttt{02}), \neg inFront(\ttt{blue}, \ttt{02}),\\ 
            &exitCounter(\ttt{02}) \pluseq 1, \neg occupied(\ttt{101})}.
    \end{flalign*}
    Intuitively, \ttt{blue} can exit from signal $\ttt{02}$ if it is in front of it and the priority is respected, then it is signalled that \ttt{blue} has left the station from \ttt{02}, and the last track circuit of \ttt{23-02} is freed.
\end{description}
\end{description}
\noindent A possible valid plan for $\Pi$, assuming $\epsilon = 0.001$, is
{\small
\begin{flalign}
    \pi = \set{& \tuple{5.001, move(\ttt{blue}, \ttt{03-21}), 30.0}, \tuple{35.002, stop(\ttt{blue}, \ttt{21}, \ttt{I}), 10.0},\nonumber \\
& \tuple{50.001, depart(\ttt{red}, \ttt{II}, \ttt{22-02}), 10.0}, \tuple{50.002, move(\ttt{red}, \ttt{22-02}), 25.0},\nonumber \\
& \tuple{75.003, exit(\ttt{red}, \ttt{22-02}), 0.0}, \tuple{100.001, depart(\ttt{blue}, \ttt{I}, \ttt{21-02}), 10.0},\nonumber \\
& \tuple{100.002, move(\ttt{blue}, \ttt{21-02}), 15.0}, \tuple{115.003, exit(\ttt{blue}, \ttt{23-02}), 0.0}}. \label{eq:plan-motivating}
\end{flalign}
}
In the plan, $\ttt{blue}$ enters as soon it arrives at the entry point $\ttt{03}$ at $t = 5$, through route $\ttt{03-21}$ and then it stops in platform $\ttt{I}$. Since $\ttt{red}$ must exit before \ttt{blue}, it is the first to depart from platform $\ttt{II}$. Since in $[30,50]$ the track circuit $\ttt{102}$ is under maintenance, -- and all the routes directed to \ttt{02} contains it -- even if it could depart in $t=30$, \ttt{red} departs one $\epsilon$ before $t = 50$ so that it can move through \ttt{22-02} -- and occupy the just maintained track circuit \ttt{102}-- as soon as possible. After, \ttt{red} has exited, \ttt{blue} can start to move towards the exit as soon as its depart time in $t=100$ has been reached.

\section{Symbolic Pattern Temporal Numeric Planning with \ICEs} \label{sec:approach}
In this section, we present the main contribution of this paper, i.e., a Symbolic Pattern Planning encoding for Temporal Numeric Planning with \ICEs. Firstly, in Section \ref{sec:rolling} we present the concept of \textsl{rolling} durative actions and when an action can be rolled, then, in Section \ref{sec:pattern} we explore the concept of a pattern and how it can be formally extended and computed in the presence of \ICEs. In Section \ref{sec:spp} we propose our Planning as Satisfiability encoding to deal with both the causal and temporal aspects of temporal planning with \ICEs. Finally, in Section \ref{sec:correctness-completeness} we prove the correctness and completeness of our approach. Throughout these sections, our \instradi motivating example will illustrate and ground our approach.

\subsection{Rolling Durative Actions} \label{sec:rolling}
The concept of \textsl{rolling} was firstly introduced by \citet{Scala_Haslum_Thiebaux_Ramirez_2016_AIBR} for Numeric Planning and then exploited in \spp both for Numeric Planning \cite{DBLP:conf/aaai/CardelliniGM24} and for Temporal Numeric Planning \cite{Cardellini_Giunchiglia_2025}. When actions have numeric effects, consecutively repeating the same action may lead to different states, depending on the number of repetitions. For example, the value of a numeric variable $x$ after the (possibly repeated) execution of an action having $x \pluseq 1$ in its effects depends on how many times the action was executed. In standard \pas approaches \cite{DBLP:conf/ausai/WehrleR07,balyo_relaxing_2013,  bofill_espasa_villaret_2016,Panjkovic_Micheli_2023,Panjkovic_Micheli_2024}, an action can be applied at most once per step, and thus, a plan with $r$ consecutive repetitions of the action will be found in a number of steps $n \geq r$. The \textsl{rolling} technique, allows modelling consecutive repetitions of the action in a single step, thus reducing the bound $n$.  This approach has been very beneficial in experimental analysis for both Numeric Planning and Temporal Numeric Planning.

Intuitively, in the rolling technique (which will see in detail later) we encode (or \textsl{roll}) the action execution via an integer variable, denoting how many times the action is (possibly) consecutively repeated. We can then express the action's effects as a function of this variable. Not all actions can, however, be rolled. Intuitively, a durative action $b$ can be consecutively executed more than once when $(i)$ its propositional \IEs do not disable the repetition of $b$, given its \ICs,  $(ii)$ the numeric \IEs do not interfere between themselves, and $(iii)$ it might be useful to execute $b$ more than once. Formally,  we say that 
$b = \tuple{\icond(b), \ieff(b), [L,U]}$ is {\sl eligible for rolling} if the following three conditions are satisfied:
\begin{enumerate}
        \item if $v := \top$ (resp. $v := \bot$) is in the effects of any \IE of $\ieff(b)$ then $v = \bot$ (resp. $v = \top$) is not in the conditions of any \IC of $\icond(b)$;
    \item if $x \asseq \psi$ is a numeric effect of any \IE in $\ieff(b)$, then 
    \begin{enumerate}
        \item $x$ does not occur in the r.h.s. of any other \IE in $\ieff(b)$,
        \item either $x$ does not occur in $\psi$ or $x \asseq \psi$ is a linear increment,
        \item if $x := \psi$ is not a linear increment, then another \IE assigning $x$ does not appear in any \IE of $b$,
     \end{enumerate}
    \item $\ieff(b)$ includes an \IE with a linear increment as effect (i.e., rolling is useful).
 \end{enumerate}
If $b$ is eligible for rolling, consecutively executing $b$ for $r\ge 1$ times:
\begin{enumerate}
    \item The total duration of the $r$ repetitions is in \begin{equation}
        [r \times L + (r-1) \times \epsilon_b, r \times U + (r-1) \times \epsilon_b], \label{eq:rolling-duration}
    \end{equation}
    where
    \begin{enumerate}
        \item $\epsilon_b = \epsilon$, if there exists an \IE $e = \tuple{\tau, \eff(e)}$ in $\ieff(b)$ with $\tau = \tend$ in mutex with either $(i)$ an \IC $c = \tuple{\tau^\vdash, \tau^\dashv, \cond(c)}$ in $\icond(b)$ with $\tau^\vdash = \tstart$, or $(ii)$ another \IE $e' = \tuple{\tau', \eff(e')}$ in $\ieff(b)$ with $\tau' = \tstart$,
        \item $\epsilon_b =0$ otherwise. 
    \end{enumerate}
    Such interval respects the $\epsilon$-separation property for plan validity if an \ICE in \tstart is in mutex with an \ICE in \tend.
    \item Causes $v$ to get value $(r \times \psi)$  if $v \pluseq \psi$ is a linear increment effect of any $\IE$ of $\ieff(b)$, while all the other  variables keep the value they get after the first execution of $b$.
\end{enumerate}
Notice that it is assumed that all the consecutive executions of $b$ have the same duration. Indeed, according to the semantics, the duration of $b$ can be arbitrarily fixed as long as each single execution respects the duration constraints, which are part of the domain specification. We will see later how this assumption does not affect the completeness of our encoding since a valid plan --assuming it exist-- can be found even if rolling is disabled, though in a greater or equal number of steps. 

\subsection{The Pattern} \label{sec:pattern}
The notion of pattern was introduced, originally, by \citet{DBLP:conf/aaai/CardelliniGM24} for Numeric Planning, where it was defined as a finite sequence of (snap) actions. Intuitively, its purpose was to suggest a causal order between actions, which a \pas planner can exploit to find a valid plan at lower bounds than any of the other  \pas approaches in the literature.  The notion of pattern is extended to Temporal Numeric Planning by \citet{Cardellini_Giunchiglia_2025}, where a pattern is defined as a finite sequence of snap actions representing durative actions' starts and ends. In this case, its purpose is to suggest both a casual and interleaving order between durative actions' starts and ends. Introducing \ICEs, we also need to account for the causal and interleaving order between action/plan-\ICEs. In Section \ref{sec:pattern-definition}, we formally define a pattern in the case of Temporal Numeric Planning with \ICEs. In Section \ref{sec:pattern-computation} we show how a pattern can be computed.

\subsubsection{Pattern Definition} \label{sec:pattern-definition}
Let $\Pi = \tuple{X, A, I, G, C, E}$ be a planning task. A {\sl pattern} $\pattern = h_1;\dots;h_k$, with length $k \geq 0$, is a finite sequence of \textsl{happenings} of $\Pi$. A happening $h$ is a tuple $\tuple{\cond(h), \eff(h)}$ where $\cond(h)$ is a (possibly empty) set of propositional or numeric conditions and $\eff(h)$ is a (possibly empty) set of propositional or numeric effects. Intuitively, we can model: $(i)$ an \IC $c = \tuple{\tau^\vdash, \tau^\dashv, \cond(c)}$ into a happening $h$ with $\cond(h) = \cond(c)$ and $\eff(h) = \emptyset$, and $(ii)$ an \IE $e = \tuple{\tau, \eff(e)}$ into a happening $h'$ with $\cond(h') = \emptyset$ and $\eff(h') = \eff(e)$. 

A happening could correspond to multiple \IC and \IE with the same conditions and effects, respectively. For these reasons, we perform the following two initial steps, which do not affect the generality of our approach:
\begin{enumerate}
    \item Whenever in $\icond(\Pi)$ there are two distinct \ICs $c_1 = \tuple{\tau^\vdash_1, \tau^\dashv_1, \icond(c_1)}$ and $c_2 = \tuple{\tau^\vdash_2, \tau^\dashv_2, \icond(c_2)}$ with $\tau^\vdash_1 = \tau^\vdash_2$, $\tau^\dashv_1 = \tau^\dashv_2$, and $\icond(c_1) = \icond(c_2)$, we break the identity by adding to the conditions of one of the two \ICs an always satisfied condition, like $k=k$ for some $k \in \rational$.
    \item Whenever in $\ieff(\Pi)$ there are two distinct \IEs $e_1 = \tuple{\tau_1, \eff(e_1)}$ and $e_2 = \tuple{\tau_2, \eff(e_2)}$ with $\tau_1 = \tau_2$ and $\eff(e_1) = \eff(e_2)$, we break the identity by changing an effect $x := \psi$ in one of the two \IEs with the equivalent $x := \psi + k - k$, for some $k \in \rational$.
    \end{enumerate}
With this disambiguation, given a happening $h$, in the following we will denote with $h=e$ if $h$ has the same effects of the \IE $e$ and with $h=c$ if $h$ has the same conditions of the \IC $c$. We will also say that $h \in \icond(\Pi)$ (resp. $h \in \ieff(\Pi)$) if there is a $c \in \icond(\Pi)$ (resp. $e \in \ieff(\Pi)$) such that $h=c$ (resp. $h=e$) and we naturally extend this to set operations. A pattern $\pattern = h_1;\dots;h_k$ thus intuitively represents the expected interleave of plan/action-\ICs and plan/action-\IEs. 

We remember that, by construction, in Sec \ref{sec:ices-formal}, we assumed to always have at least one \IC or \IE both at the start and at the end of $b$. Thus, slightly abusing notation, we will denote a happening $h$ in a pattern which represent the start or end of an action $b$ as $h = b^\vdash$ or $h = b^\dashv$, respectively.

A pattern is arbitrary, allowing for multiple occurrences of the same happening, even consecutively. 
Though the pattern can contain multiple occurrences of the same happening $h$, such occurrences will be treated as different copies of $h$, allowing us to treat each happening occurrence in the pattern
as a variable in our encoding, simplifying the notation and the presentation. 
For this reason, for each copy of the same happening $h$ in the pattern, we can replace it with a distinct copy $h'$, disambiguating each copy as previously discussed.
We can therefore take the happenings in the pattern to be the action variables in our encoding, and we can assume that a happening $h_i = b^\vdash$/$h_i = b^\dashv$  starts/ends exactly one durative action in $A$. 

We say that a pattern $\pattern = h_1;\dots;h_k$ is \textsl{complete} iff $(i)$ for each $c \in \icond(\Pi)$ (Eq. \ref{eq:icondPi}) there is at least one $i \in [1,k]$ such that $h_i = c$, and $(ii)$ for each $e \in \ieff(\Pi)$ (Eq. \ref{eq:ieffPi}) there is at least one $i \in [1,k]$ such that $h_i = e$.

In the \spp approach for temporal numeric planning without \ICEs \cite{Cardellini_Giunchiglia_2025}, a pattern was defined as a finite sequence of durative actions' start and ends. This is because, without \ICEs, conditions can be imposed and effects can be executed only at the start and at the end of actions and thus, what we now call happenings, coincide exactly with durative actions' start and ends.

\subsubsection{Pattern Computation} \label{sec:pattern-computation}
We now illustrate how to compute a pattern. In \citet{DBLP:conf/aaai/CardelliniGM24}, for Numeric Planning, the pattern was computed using the \textsl{Asymptotic Relaxed Planning Graph (\arpg)}, firstly introduced by \citet{Scala_Haslum_Thiebaux_Ramirez_2016_AIBR}. In this paper, we show how the \arpg can be employed also for Temporal Numeric Planning with \ICEs.

\paragraph{Relaxed State} Let $\Pi = \tuple{X, A, I, G, C, E}$ be a planning task with $X = \tuple{V_B, V_N}$. A \textsl{relaxed state} $\hat{S}$ is
a function mapping each propositional variable $v \in V_B$ to a subset of $\set{\bot, \top}$ and each numeric variable $x \in V_N$ to an interval 
$[\lb{x}, \ub{x}]$ with $\lb{x} \in \rational \cup \set{-\infty}$, $\ub{x} \in \rational \cup \set{+\infty}$, $\lb{x} \le \ub{x}$.  Intuitively, 
the set associated to each variable represents the values it can assume in the relaxed state. 
A relaxed state $\hat{S}$ can be extended to linear numeric expressions $\psi$ by defining  $\hat{S}(\psi)$ according to Moore's Interval Analysis \cite{moore2009introduction}:  
for  $x,y \in V_N$, $c \in \rational$,  $\hat{S}(x) = [\lb{x}, \ub{x}]$ and  $\hat{S}(y) = [\lb{y}, \ub{y}]$, we define:
\begin{gather*}
    \hat{S}(c) = [c, c], \qquad \hat{S}(x + c) = [\lb{x} + c, \ub{x} + c],\\
    \hat{S}(x + y) = [\lb{x} + \lb{y}, \ub{x} + \ub{y}], \qquad \hat{S}(x - y) = [\lb{x} - \ub{y}, \ub{x} - \lb{y}],\\
    \hat{S}(cx) = [\min(c \lb{x}, c  \ub{x}), \max(c  \lb{x}, c \ub{x})].
\end{gather*}
The \textsl{convex union} of two intervals $[\lb{x}, \ub{x}]$ and $[\lb{y}, \ub{y}]$ is denoted and computed as $[\lb{x}, \ub{x}] \sqcup [\lb{y}, \ub{y}] = [\min(\lb{x}, \lb{y}), \max(\ub{x}, \ub{y})]$.
The {\sl convex union of two relaxed states} $\hat{S}_1$ and $\hat{S}_2$ is the relaxed state $\hat{S} = \hat{S}_1 \sqcup \hat{S}_2$  such that
\begin{enumerate}
    \item for each $v \in V_B$, 
$\hat{S}(v) = \hat{S}_1(v) \cup \hat{S}_2(v)$, and 
\item for each $x \in V_n$, $\hat{S}(x) = \hat{S}_1(x) \sqcup \hat{S}_2(x)$. 
\end{enumerate}
Consider a relaxed state $\hat{S}$.
We say that $\hat{S}$ {\sl satisfies}
\begin{enumerate}
    \item {\sl a propositional condition} $v = \top$ if $\top \in \hat{S}(v)$, and $v = \bot$ if $\bot \in \hat{S}(v)$,
    \item {\sl a numeric condition} $\psi \unrhd 0$ if $\hat{S}(\psi) = [\lb{\psi}, \ub{\psi}]$ and $\ub{\psi} \unrhd 0$.
\end{enumerate}    
If $\hat{S}$ satisfies all the conditions in a set $\Gamma$, then $\hat{S}$ satisfies $\Gamma$, written $\hat{S} \models \Gamma$, 

 \paragraph{Executing a Sequence of Snap Actions in a Relaxed State}
Consider a single snap action $a = \tuple{\pre(a), \eff(a)}$ and a relaxed state $\hat{S}$. In order to define what is the relaxed state resulting from the execution of an action $a$ in a relaxed state, we consider the special case in which, for each numeric effect $x \asseq \psi \in \eff(a)$, either $\psi$ does not contain any variable assigned by an action in $A$ or $\psi$ is expressed as $(x + \varphi)$, with $\varphi$ a linear expression over $X$. Such an assumption is not a restriction, since \cite{Scala_Haslum_Thiebaux_Ramirez_2016_AIBR} $$\psi = \sum_{y \in V_N} k_y y + k = x + \varphi \quad\text{with}\quad \varphi = \sum_{y \in V_N, y \neq x} k_y y + (k_x - 1) x + k.$$ 

Consider a relaxed state $\hat{S}$.
An  action $a$ is {\sl  executable} (or {\sl applicable}) in $\hat{S}$ if $\hat{S}$ satisfies the preconditions of $a$. The result of executing $a$ in $\hat{S}$ is the relaxed state $\hat{S}' = res(\hat{S},a)$ such that:
\begin{enumerate}
    \item $\hat{S}'$ is undefined if $\hat{S} \not \models \pre(a)$.
    \item For each propositional variable $v \in X$, $(i)$ if $v := \top \in \eff(a)$, then $\hat{S}'(v) = \hat{S}(v) \cup \set{\top}$  $(ii)$ if  $v := \bot \in \eff(a)$, then $\hat{S}'(v) = \hat{S}(v) \cup \set{\bot}$ and $(iii)$ $\hat{S}'(v) = \hat{S}(v)$ otherwise. 
    \item For each numeric variable $x$ with $x \asseq \psi \in \eff(a)$,
    \begin{enumerate}
        \item 
         if $\psi$ does not contain a variable assigned by an action in $A$, then 
         \begin{equation}
             \hat{S}'(x) = \hat{S}(x) \sqcup \hat{S}(\psi), \label{eq:relaxed-convex-general-assignment}
         \end{equation}
        \item otherwise, if $\psi =  x + \varphi$, then $\hat{S}'(x) = [\lb{x}',\ub{x}']$, where, given $\hat{S}(x) = [\lb{x},\ub{x}]$ and $\hat{S}(\varphi) = [\lb{\varphi},\ub{\varphi}]$,
    \begin{flalign*}
        \lb{x}' = \begin{cases}
        - \infty & \text{if } \lb{\varphi} < 0,\\ 
        \lb{x} & \text{otherwise},
    \end{cases} \quad
    \ub{x}' = \begin{cases}
        +\infty & \text{if } \ub{\varphi} > 0,\\ 
        \ub{x}& \text{otherwise}.
    \end{cases}
    \end{flalign*}
    \item For each numeric variable $x$ not assigned by $a$, $\hat{S}'(x) = \hat{S}(x)$. 
    \end{enumerate} 
    
\end{enumerate}
Finally, let $\mathit{A}$ be a set of snap actions. The {\sl parallel execution} of the actions in $A$ results in the relaxed state 
$$\hat{S}' = res(\hat{S}, A) = \bigsqcup_{a \in A} res(\hat{S}, a).$$
Let $A_1;\dots;A_k$ be a sequence of snap actions' set. Applying the sequence from an initial relaxed state $\hat{S}_0$ induces a $k+1$-long sequence $\hat{S}_0;\hat{S}_1;\dots;\hat{S}_k$ such that $\hat{S}_{i} = res(\hat{S}_{i-1}, A_{i})$ for $i \in \set{1,2,\dots,k}$. Thus, the resulting state applying the sequence to $\hat{S}_0$ is $\hat{S}_k$, i.e., $res(\hat{S}_0, A_1;\dots;A_k) = \hat{S}_k$.

\newcommand{\AICEs}[1]{\textsc{aiceSeq}(#1)}
\newcommand{\AIEs}[1]{\textsc{aieSeq}(#1)}
\newcommand{\PICEs}[1]{\textsc{piceSeq}(#1)}

\renewcommand{\TH}{\mathit{TH}}
\newcommand{\CTH}{\mathit{CTH}}

\paragraph{Compilation of Durative Actions into Snap Actions} Let $\Pi = \tuple{X,A,I,G,C,E}$ be a planning task. Before computing the \arpg, we first compile the planning task $\Pi$ in a planning task $\snapPi$, denoted as the \textsl{snap planning task} of $\Pi$. Intuitively, the compilation to $\snapPi$ transforms all $\ICs$ and $\IEs$ in snap (instantaneous) actions, removing the temporal aspects from the planning task, allowing us to make use of the \arpg computation already used by \citet{DBLP:conf/aaai/CardelliniGM24} and \citet{Scala_Ramirez_Haslum_Thiebaux_2016_Rolling}.

For each durative action $b = \tuple{\icond(b), \ieff(b), [L,U]} \in A$, we can construct a set of \textsl{timed happenings of $b$}. A timed happening is a tuple $\tuple{t^\vdash, t^\dashv, h}$ with $t^\vdash, t^\dashv \in \rational^{> 0}$ absolute times and $h$ a happening. Intuitively, the set can be constructed as if $b$ lasted $L$ (Eq. \ref{eq:tau_a_b}), i.e.,
    \begin{flalign}
        \TH^b &= \set{\tuple{\tau^\vdash[0,L], \tau^\dashv[0,L], \tuple{\cond(c), \emptyset}} \mid c = \tuple{\tau^\vdash, \tau^\dashv, \cond(c)} \in \icond(b)}\nonumber\\ 
        & \cup ~ \set{\tuple{\tau[0,L], \tau[0,L], \tuple{\emptyset, \eff(e)}} \mid e = \tuple{\tau, \eff(e)} \in \ieff(b)}. \label{eq:aices-order}
    \end{flalign}
    From the set $\TH^b$, we perform the following operations:
    \begin{enumerate}
        \item Firstly, for each pair of timed happenings of $b$ scheduled at the same time, i.e., $\tuple{t, t, \tuple{\cond(h_1), \emptyset}}$ and $\tuple{t, t, \tuple{\emptyset, \eff(h_2)}}$, we modify the latter to be $\tuple{t, t, \tuple{\cond(h_1), \eff(h_2)}}$. In this way, the happening associated to an $\IE$ at time $t$ can be applied only if the happening of the $\IC$ at time $t$ is also applied;
        \item Then, from the modified $\TH^b$ we construct a sequence of pairwise disjoint set of combined timed happenings $\TH^b_1;\dots;\TH^b_p$, such that $\bigcup_{i=1}^p \TH^b_i = \TH^b$, obtained by partitioning the set $\TH^b$ into subsets where, for each $\TH^b_i$, each $\tuple{t^\vdash, t^\dashv, h} \in \TH^b_i$ has the same\footnote{In a single set $\TH^b_i$ we thus have the \ICs of $b$ with the same $t^\vdash$ but with different $t^\dashv$.} $t^\vdash$ and order the sets by $t^\vdash$ into the sequence $\TH^b_1;\dots;\TH^b_p$;
        \item Finally, from the sequence $\TH^b_1;\dots;\TH^b_p$ we obtain the sequence of set of happenings $H^b_1;\dots;H^b_p$ -- denoted as \AICEs{b}, for \textsl{action \ICEs} -- by \textsl{projection}, i.e., $H_i^b = \set{h \mid \tuple{t^\vdash, t^\dashv, h} \in \TH_i^b}$.
    \end{enumerate}

\newcommand{\exec}{\ttt{exec}}
\newcommand{\Exec}{\ttt{Exec}}
\noindent Let $\Pi = \tuple{X, A, I, G, C, E}$ be a planning task. The planning task $\snapPi$ is then $\tuple{X_\snap, A_\snap, I_\snap, G_\snap, \emptyset, \emptyset}$ where:
\begin{enumerate}
    \item $X_\snap$ contains the variables of $X$, adding 
    \begin{enumerate}
        \item a propositional variable $\exec(h)$  for each \ICE $h \in \icond(\Pi) \cup \ieff(\Pi)$, and
        \item a numeric variable $\ttt{time}$. 
    \end{enumerate}
    \item The set of snap actions $A_\snap$ is constructed as follows: 
    \begin{enumerate}
        \item Let $\AICEs{b} = H_1^b;\dots;H_p^b$ be the sequence of sets of happenings of action $b$, as discussed previously. We can construct the set of snap actions $A_\snap^{\AICEs{b}}$ as the union of the sets, for $i \in [1,p]$,
        \begin{gather}
            \set{\tuple{\cond(h) \cup \Exec(H_{i-1}^b), \eff(h) \cup \set{\exec(h)}} \mid h \in H_i^b}, \label{eq:snap-compilation-action}
        \end{gather}
        where $H_0^b = \emptyset$, and $\Exec(H_{i-1}^b)$ is the abbreviation for 
        $$\Exec(H_{i-1}^b) = \set{\exec(h) \mid h \in H_{i-1}^b}.$$
        Intuitively, $A_\snap^\AICEs{b}$ contains one snap action for each \ICE of $b$, with the same conditions and effects, and where the variables $\exec(\cdot)$ impose the order given by $\AICEs{b}$;
        \item Let $M$ be an integer, larger than the maximum offset of the \ICEs in $C \cup E$, we can construct the set of actions $A_\snap^M$ as 
        \begin{gather}
            A_\snap^M = \set{\tuple{\set{\ttt{time} = i-1}, \set{\ttt{time} := i}} \mid i \in [1, M]}. \label{eq:snap-compilation-M}
        \end{gather}
        Intuitively, in the relaxed initial state $\hat{I}_\snap$ we will have $\hat{I}_\snap(\ttt{time}) = [0,0]$ and each action in $A_\snap^M$ will "move the time forward", with the convex union of Eq. \ref{eq:relaxed-convex-general-assignment}, increasing the interval of $\hat{I}_\snap(\ttt{time})$;
        \item We denote as $A_\snap^{C,E}$ the set containing, for each \IC $c = \tuple{\tau^\vdash, \tau^\dashv, \cond(c)} \in C$ and each \IE $e = \tuple{\tau, \eff(e)} \in E$, the snap actions
        \begin{flalign}
            \tuple{\cond(c) \cup \set{\ttt{time} = \tau^\vdash[0,M]}, \set{\exec(c)}}, \label{eq:snap-compilation-plan} \\
            \tuple{\set{\ttt{time} = \tau[0,M]}, \set{\exec(e)} \cup \eff(e)} \nonumber, 
        \end{flalign}
        respectively, where $M$ is the value described at the previous item. Intuitively, the set $A_\snap^{C,E}$ contains, for each plan-\IC and plan-\IE, an associated snap action with the same conditions and effects, executable only if the variable \ttt{time} has reached the value of time associated to the plan-\IC or plan-\IE, as if the plan lasted $M$. The $\exec(\cdot)$ variable signals the \IC or \IE has been satisfied or executed.
    \end{enumerate}
    
    Then the action in $A_\snap$ are 
    \begin{flalign}
        A_\snap = \bigcup_{b \in A} A_\snap^{\AICEs{b}} \cup A_\snap^M \cup A_\snap^{C,E}.
     \label{eq:snap-durative-2}
    \end{flalign}
    \item $s_\snap$ is the initial state $I$ extended to have $I_\snap(\ttt{time}) = 0$.
    \item $G_\snap$ is the set 
    \begin{equation}
        G_\snap = G \cup \set{\exec(h) \mid h \in C \cup E}, \label{eq:gsnap}
    \end{equation}
    ensuring that all the plan-\ICs and plan-\IEs have been satisfied and executed in the goal state.
    \item There are no plan-\ICs and plan-\IEs.
\end{enumerate}

\newcommand{\supporters}[1]{\Call{Sup}{#1}}
\newcommand{\action}[1]{\Call{SAct}{#1}}

To compute\footnote{See \cite{Aldinger_Mattmüller_Göbelbecker_2015,Scala_Haslum_Thiebaux_Ramirez_2016_AIBR} for an analysis of why supporters are necessary when computing the \arpg.} the resulting relaxed state when different actions have \textsl{interfering numeric effects} (e.g., $x := \psi$ and $y := \psi'$ with $y$ in $\psi$ and $x$ in $\psi'$) \citet{Scala_Haslum_Thiebaux_Ramirez_2016_AIBR} introduced a compilation of an action to its \textsl{supporters}. Let $a$ be an action, the set of supporters of $a$ is the set of actions
\begin{flalign*} 
    \supporters{a} = &~\set{\tuple{\pre(a) \cup \set{x < \psi}, \set{x:= \psi}}\mid x := \psi \in \eff(a)}\\
    \cup &~\set{\tuple{\pre(a) \cup \set{x > \psi}, \set{x:= \psi}} \mid x := \psi \in \eff(a)}\\ 
    \cup &~\set{\tuple{\pre(a), \set{x := \top \in \eff(a)} \cup \set{x := \bot \in \eff(a)}}}.
\end{flalign*}

Intuitively, for each action and for each effect $x:=\psi \in \eff(a)$ we split the action in two: one in the case $x < \psi$ and one in the case $x > \psi$. We also add a supporter containing all the propositional effects of $a$. If $a' \in \supporters{a}$, we denote by $\action{a'}$ the action $a$. If $A$ is a set of actions, then $\supporters{A} = \bigcup_{a \in A} \supporters{a}$. 

\newcommand{\aleft}{A_\mathrm{left}}
\begin{algorithm}[t]
\caption{\textsc{ComputeSnapARPG}. \\
Input: the planning task $\Pi = \tuple{X,A,I,G,C,E}$. \\
Output: an Asymptotic Relaxed Planning Graph (\arpg). \label{alg:compute-arpg}
}

{\small
\begin{algorithmic}[1]
\Function{$\textsc{ComputeSnapARPG}$}{$\Pi$}
    \State $\arpg \gets \epsilon$ 
    \State $\tuple{X_\snap, A_\snap, I_\snap, G_\snap, \emptyset, \emptyset} \gets \snapPi$ \label{alg:compute-arpg:tosnap}
    \State $\aleft \gets \supporters{A_\snap}$ \label{arg:compute-arpg:sup}
    \State $\hat{S} \gets \Call{Relax}{I}$
    \While{\textsc{True}} \label{alg:compute-arpg:while}
        \State $A_\layer \gets \set{a \mid a \in \aleft, \text{$\hat{S} \models \pre(a)$}}$ \label{alg:compute-arpg:alayer}
        \If{$A_\layer = \emptyset$} \label{alg:compute-arpg:first-empty}
            \State \Return $\arpg$ \label{alg:compute-arpg:return-arpg}
                    
        \EndIf
        \State $A_\layer' \gets \set{\action{a} \mid a \in A_\layer} \setminus \set{\action{a} \mid a \in \supporters{A_s} \setminus \aleft}$ \label{alg:compute-arpg:originating-action}
        \State $\arpg \gets \arpg;A'_\layer$ \label{alg:compute-arpg:concat}
        \State $\hat{S} \gets res(\hat{S}, A_\layer)$ \label{alg:compute-arpg:res}
        \State $\aleft \gets \aleft \setminus A_\layer$ \label{alg:compute-arpg:asnap-minus-alayer}

    \EndWhile
\EndFunction
\end{algorithmic}
}
\end{algorithm}

\paragraph{\arpg computation for Temporal Planning} Let $\Pi = \tuple{X,A,I,G,C,E}$ be a planning task. We now show how to compute an \textsl{Asymptotic Relaxed Planning Graph (\arpg)} for $\Pi$. Formally, an \arpg is a sequence of \textsl{layers} $A_1;\dots;A_l$ of length $l$. A layer is a set of snap actions. Intuitively, an \arpg provides a partial causal order between snap actions, suggesting that an action in layer $A_p$ must happen before a action in layer $A_q$, if $p < q$. The \arpg for a temporal planning task $\Pi$ can thus be constructed from $\snapPi$, as showed in Algorithm \ref{alg:compute-arpg}:
\begin{enumerate}
    \item Firstly, we initialize $(i)$ the $\arpg$ as the empty sequence $\epsilon$, $(ii)$ we compile $\Pi$ in $\snapPi$, $(iii)$ we denote as $\aleft$ all the supporters of $A_\snap$, and $(iv)$ we initialize $\hat{S}$ as the relaxation of $I$.
    \item Then, in Line \ref{alg:compute-arpg:alayer} we denote as $A_\layer$ all the (supporter) snap actions in $A_\snap$ which are applicable in the relaxed state $\hat{S}$.
    \item If the layer is not empty, we move to Line \ref{alg:compute-arpg:originating-action}, $(i)$ adding to the \arpg the layer containing all the actions of the supporters in $A_\layer$ which have not been inserted before in the \arpg, $(ii)$ updating the state $\hat{S}$, $(iii)$ removing the used actions in $A_\layer$ from $\aleft$, and then starting again from Line \ref{alg:compute-arpg:alayer} and the point above.
    \item If instead, in Line \ref{alg:compute-arpg:first-empty}, $A_\layer$ is empty, it means that either $(i)$ we have no more actions in $A_\snap$ or $(ii)$ none of them are applicable in $\hat{S}$. Thus, we end the algorithm and return the \arpg.
\end{enumerate}

One of the properties of the \arpg for numeric planning, is that if, the final relaxed state induced by the \arpg doesn't satisfy the goal, then it doesn't exist a valid plan for the planning task. We will prove that this is also true, in some cases, for temporal numeric planning with \ICEs. To explain in which cases this property holds, we introduce the concept of \textsl{well-orderability}. Let $b = \tuple{\icond(b), \ieff(b), [L,U]}$ be a durative action. \textsl{An action $b$ is well-orderable} if $(i)$ the duration is fixed, i.e., $L=U$, or $(ii)$ we have that $\max(K_b^\tstart) + \max(K_b^\tend) < L$, with $K_b^\tstart$ and $K_b^\tend$ defined as
\begin{equation}
\begin{array}{@{}r@{}r@{}l}
K^\tstart_b &= &\set{k \mid \tuple{\tau^\vdash, \tstart + k, \cond(c)} \in \icond(b)}\\
   & \cup &\set{k \mid \tuple{\tstart + k, \tau^\dashv, \cond(c)} \in \icond(b)} \\
   & \cup & \set{k \mid \tuple{\tstart + k, \eff(e)} \in \ieff(b)},\\
K^\tend_b &= &\set{k \mid \tuple{\tau^\vdash, \tend - k, \cond(c)} \in \icond(b)}\\
   & \cup &\set{k \mid \tuple{\tend - k, \tau^\dashv, \cond(c)} \in \icond(b)} \\
   & \cup & \set{k \mid \tuple{\tend - k, \eff(e)} \in \ieff(b)},
\end{array}\label{eq:K_b}
\end{equation}
i.e., the sets of the offsets of the happenings of $b$ with time relative from the start and end, respectively.
Intuitively, a durative action is well-orderable if the order of its \ICEs is independent of its duration.

Let $\Pi = \tuple{X,A,I,G,C,E}$ be a planning task. We denote by $K^\talpha_\Pi$ and $K^\tomega_\Pi$ the set of the offsets of $C \cup E$ such that: 
\begin{equation*}
\begin{array}{@{}r@{}r@{}l}
K^\talpha_\Pi ~&= &~\set{k \mid \tuple{\tau^\vdash, \talpha + k, \cond(c)} \in C}\\
   & \cup &~\set{k \mid \tuple{k, \tau^\dashv, \cond(c)} \in C} \\
   & \cup &~\set{k \mid \tuple{\talpha + k, \eff(e)} \in E}),\\
K^\tomega_\Pi ~&= &~\set{k \mid  \tuple{\tau^\vdash, \tomega - k, \cond(c)} \in C}\\
   & \cup &~\set{k \mid \tuple{\tomega - k, \tau^\dashv, \cond(c)} \in C} \\
   & \cup &~\set{k \mid \tuple{\tomega - k, \eff(e)} \in E}).
\end{array}
\end{equation*}
\textsl{A planning task $\Pi$ is well-orderable} if both $(i)$ all durative actions in $A$ are well-orderable, and $(ii)$ either $K^\talpha_\Pi = \emptyset$ or $K^\tomega_\Pi = \emptyset$. Intuitively, condition $(ii)$ guarantee that, in any plan, the order of plan-\ICs in $C$ and plan-\IEs in $E$ is independent of the length of the plan.

\begin{theorem} \label{thm:arpg-plan-existance}
    Let $\Pi$ be a well-orderable planning task. If it doesn't exist a valid plan for $\snapPi$, then it doesn't exist one for $\Pi$.

    \begin{proof}
        Let $\snapPi = \tuple{X_\snap, A_\snap, I_\snap, G_\snap, \emptyset, \emptyset}$. For contrapositive, we prove that if a valid plan for $\Pi$ exists, then there exists a valid plan for $\snapPi$. Let $\pi$ be a valid plan for $\Pi$. From $\pi$ we can construct a plan $\pi_\snap$ for $\snapPi$ as 
        \begin{gather*}
            \pi_\snap = \set{\tuple{t^\vdash, a_c} \mid \tuple{t^\vdash, t^\dashv, c} \in \icond(\pi)} \cup \set{\tuple{t, a_e} \mid \tuple{t, e} \in \ieff(\pi)}
        \end{gather*}
        where $\icond(\pi)$ and $\ieff(\pi)$ are the set described in Eqs. \ref{eq:icond} and \ref{eq:ieff} and $a_c$ and $a_e$ are the snap actions in $A_s$ of $\snapPi$ corresponding to the happening $c$ and $e$ respectively, according to Eqs. \ref{eq:snap-compilation-action}. It is clear, by construction and since $\pi$ is a valid plan for $\Pi$, that the sequences of states $I(\pi)$ and $\hat{I}_\snap(\pi_\snap)$ coincide, if we remove from each state in $\hat{I}_\snap(\pi_\snap)$ the additional $\exec(\cdot)$ and $\ttt{time}$ variables in $X_\snap$. We now prove that $\pi_\snap$ is a valid plan for $\snapPi$:
        \begin{enumerate}
            \item \textsl{initial condition}: By construction of $I_\snap$.
            \item \textsl{goal condition}: Recalling Eq. \ref{eq:gsnap}, the goal $G_\snap$ is composed of $(i)$ the original goal $G$ of $\Pi$, and $(ii)$ the goal $\set{\exec(h) \mid h \in C \cup E}$. The first part is respected  by construction of $\pi_\snap$. The second part of the goal is respected because $(i)$ each \IC in $C$ is satisfied by the states in $I(\pi)$, and thus, by construction, is respected\footnote{Notice that while in $\pi$ an \IC $\tuple{t^\vdash, t^\dashv, \cond(c)}$ must respect $\cond(c)$ from $t^\vdash$ to $t^\dashv$ (extrema included), in $\pi_\snap$ we have to respect it only on $t^\vdash$.} by $I(\pi_s)$, $(ii)$ each plan-$\IE$ in $E$ is present in $\ieff(\pi)$ and thus is in $\pi_\snap$ by construction, and $(iii)$ being $\Pi$ well-orderable, the order of the plan-\ICs and plan-\IEs' in $\pi$ corresponds with the order of the associated actions $A_\snap^M$ in $\pi_\snap$, since $M$ is larger than the maximum offset of the \ICEs in $C \cup E$, and having that either $K^\talpha_\Pi = \emptyset$ or $K^\tomega_\Pi = \emptyset$. Thus, by all the $\exec(h)$ with $h \in C \cup E$ must be respected in $G_\snap$.
            \item \textsl{intermediate conditions}: For each $c \in \icond(\Pi)$ we have a snap action $a_c \in A_\snap$ (Eqs. \ref{eq:snap-compilation-action} and \ref{eq:snap-compilation-plan}) where $\pre(a_c) = \cond(c) \cup \Gamma_c$. The conditions in $\cond(c)$ are respected by construction of $\pi_s$. $\Gamma_c$ can either $(i)$ concern the variable $\ttt{time}$ if $c$ is a plan-\IC or $(ii)$ concern some variables $\exec(\cdot)$ if $c$ is an action-\IC. The case $(i)$ was already covered by the previous item. For case $(ii)$, by hypothesis, $\Pi$ is well-orderable, i.e., for each action $b \in \Pi$ the order of happenings in $\AICEs{b}$ and the ordering of the corresponding actions in $\pi_\snap$ coincides. The order causes the satisfaction of the added conditions concerning $\exec(\cdot)$. 
            \item \textsl{self-overlapping} and \textsl{$\epsilon$-separation}: by construction of $\pi_\snap$, being $\pi$ valid for $\Pi$ and being all the actions in $A_\snap$ instantaneous.
        \end{enumerate}
        \vspace{-25pt}
    \end{proof}
\end{theorem}

\begin{corollary} \label{cor:compute-snap-arpg}
    Let $\Pi = \tuple{X,I,A,G,C,E}$ be a well-ordered task, $\hat{S} = \Call{Relax}{I}$ and $\arpg = \Call{ComputeSnapARPG}{\Pi}$. If $res(\hat{S}, \arpg) \not \models G$ then it doesn't exist a valid plan for $\Pi$.
    \begin{proof}
        The procedure $\Call{ComputeSnapARPG}{\Pi}$ in Alg. \ref{alg:compute-arpg} corresponds to constructing an \arpg for $\Pi_\snap$. It was already proved by \citet{Scala_Haslum_Thiebaux_Ramirez_2016_AIBR} that if the final relaxed state induced by an \arpg doesn't satisfy the goal, then it doesn't exist a valid plan for $\Pi_\snap$. Since it doesn't exist a plan for $\Pi_\snap$, it doesn't exist a valid plan for $\Pi$ (Thm. \ref{thm:arpg-plan-existance}).
    \end{proof}
\end{corollary}

\paragraph{From an \arpg to a Pattern} Let $\Pi = \tuple{X,A,I,G,C,E}$ be a planning task. The \arpg obtained from the procedure $\Call{ComputeSnapARPG}{\Pi}$ is a sequence of layers $A_1;\dots;A_l$ of length $l$, where each $A_i$ contains actions of $\snapPi$. From this sequence, we obtain a sequence of sets of happenings of $\Pi$, i.e., $H_1;\dots;H_l$, by swapping each snap action with its originating happening in Eqs. \ref{eq:snap-compilation-action} and \ref{eq:snap-durative-2} and not considering the snap actions in $A_\snap^M$ (Eq. \ref{eq:snap-compilation-M}). A pattern $\pattern = h_1;\dots;h_k$ is then obtained from $H_1;\dots;H_l$ by \textsl{linearization}, i.e., constructing $\pattern$ such that for each pair of happenings $h_i \neq h_j$ in $\pattern$, $h_i$ comes before $h_j$ in $\pattern$ $(i)$ if $h_i \in H_p$ and $h_j \in H_q$ with $1 \leq p < q \leq l$, else $(ii)$ if both $h_i, h_j \in H_p$ with $p \in [1,l]$, if $h_i$ is an \IC and $h_j$ is an \IE \footnote{By the semantic, we know that, if scheduled at the same time, conditions must be checked before effects are applied; thus it doesn't make sense to have \IEs before \ICs in $\prec$ if there is a possibility they are executed at the same time, e.g., they are of the same action.}. All other ties are break arbitrarily, e.g., lexicographically.
\newcommand{\Hleft}{H_\mathrm{left}}

Let $H_1;\dots;H_l$ be a sequence of set of happenings obtained by the \arpg constructed by $\Call{ComputeSnapARPG}{\Pi}$. Let 
$$\Hleft = \left(\icond(\Pi) \cup \ieff(\Pi) \right) \setminus \bigcup_{i=1}^l H_i$$
be the set of happenings of $\Pi$ which do not appear in $H_1;\dots;H_l$. We now analyse the completeness of the pattern obtained from linearization:
\begin{enumerate}
    \item If $\Hleft$ is empty, then $\prec$ is complete, by definition.
    \item Otherwise, we consider whether $\Pi$ is well-orderable or not:
    \begin{enumerate}
        \item If $\Hleft$ is not empty and $\Pi$ is well-orderable, let $h \in \Hleft$ and let $a_h = \tuple{\pre(a_h), \eff(a_h)}$ be the snap action in $\snapPi$ obtained from $h$ (Eq. \ref{eq:snap-compilation-action}). Since $a_h$ has not been included in the \arpg, it means that it doesn't exist a plan $\pi_\snap$ of $\snapPi$ that can satisfy, in the final state, $\pre(a_h)$. Thus, due to Thm. \ref{thm:arpg-plan-existance}, it doesn't exist a valid plan $\pi$ of $\Pi$ that contains the happening $h$. Thus, even if $\prec$ doesn't contain the happenings in $\Hleft$, and thus incomplete for $\Pi$, we will still consider it complete, by removing from $\Pi$ all actions containing the happenings in $\Hleft$.
        \item If $\Hleft$ is not empty and $\Pi$ is not well-orderable, then we cannot rely on Thm. \ref{thm:arpg-plan-existance} and Cor. \ref{cor:compute-snap-arpg}. In fact, if in Line \ref{alg:compute-arpg:alayer} of the procedure $\Call{ComputeSnapARPG}{\Pi}$ we have $A_\layer = \emptyset$ it could be that none of the actions left in $\aleft$ are applicable due either to the $\exec$ variables, i.e., the order imposed by $\AICEs{b}$ for some $b \in A$ or the $\ttt{time}$ variable. Thus, to obtain a complete pattern, we have to linearize
        $H_1;\dots;H_l;\Hleft$, which contains all the happenings of $\Pi$ by construction.
    \end{enumerate}
\end{enumerate}

\subsubsection{Pattern in the Motivating Example} 
Let $\Pi = \tuple{X,A,I,G,C,E}$ be the motivating example planning task. All the actions in $A$ are well-orderable because they all have a fixed duration. Even if the duration was not fixed in $[L,U]$, we would have that, for each action $b$ of the motivating example $K_b^\tend = \set{0}$, and thus $\max(K_b^\tstart) < L$ by definition. The planning task $\Pi$ is thus well-orderable because all the actions are well-orderable and $K_\Pi^\tomega = \emptyset$.

\newcommand{\tS}{\textsc{s}\xspace}
\newcommand{\tE}{\textsc{e}\xspace}
\newcommand{\tA}{\textsc{a}\xspace}
\newcommand{\tO}{\textsc{o}\xspace}

\begin{table}[]
\centering
\resizebox{\textwidth}{!}{%
{\footnotesize
\begin{tabular}{|c|c|c|c|}
\hline
$\ttt{time} = 5$&$\ttt{time} = 6$&$\ttt{time} = 7$&$\ttt{time} = 8$\\
\makecell{E[A+5]}&\makecell{move(b,03-21)[S,S]\\move(b,03-23)[S,S]\\move(b,03-21)[S]\\move(b,03-23)[S]}&\makecell{move(b,03-21)[S+5]\\move(b,03-23)[S+5]}&\makecell{move(b,03-21)[S+10]\\move(b,03-23)[S+10]}\\\hline
\hline
$\ttt{time} = 9$&$\ttt{time} = 10$&$\ttt{time} = 11$&$\ttt{time} = 12$\\
\makecell{move(b,03-21)[S+15]\\move(b,03-23)[S+15]}&\makecell{move(b,03-21)[S+20]\\move(b,03-23)[S+20]}&\makecell{move(b,03-21)[S+25]\\move(b,03-23)[S+25]}&\makecell{move(b,03-21)[E]\\move(b,03-23)[E]}\\\hline
\hline
$\ttt{time} = 13$&$\ttt{time} = 14$&$\ttt{time} = 30$&$\ttt{time} = 31$\\
\makecell{stop(b,21,I)[S,S]\\stop(b,23,III)[S,S]\\stop(b,21,I)[S]\\stop(b,23,III)[S]}&\makecell{stop(b,21,I)[E]\\stop(b,23,III)[E]}&\makecell{C[A+30,A+50]\\E[A+30]}&\makecell{depart(r,II,22-02)[S,S]\\depart(r,II,22-02)[S]}\\\hline
\hline
$\ttt{time} = 32$&$\ttt{time} = 33$&$\ttt{time} = 34$&$\ttt{time} = 35$\\
\makecell{move(r,22-02)[S,S]\\move(r,22-02)[S]}&\makecell{depart(r,II,22-02)[S+$2\epsilon$,S+$2\epsilon$]\\move(r,22-02)[S+5]}&\makecell{depart(r,II,22-02)[E]\\move(r,22-02)[S+10]}&\makecell{move(b,03-22)[S,S]\\move(r,22-02)[S+15]\\move(b,03-22)[S]}\\\hline
\hline
$\ttt{time} = 36$&$\ttt{time} = 37$&$\ttt{time} = 38$&$\ttt{time} = 39$\\
\makecell{move(b,03-22)[S+5]\\move(r,22-02)[S+20]}&\makecell{move(b,03-22)[S+10]\\move(r,22-02)[E]}&\makecell{exit(r,22-02)[S,S]\\move(b,03-22)[S+15]\\exit(r,22-02)[S]}&\makecell{move(b,03-22)[S+20]}\\\hline
\hline
$\ttt{time} = 40$&$\ttt{time} = 41$&$\ttt{time} = 42$&$\ttt{time} = 43$\\
\makecell{move(b,03-22)[S+25]}&\makecell{move(b,03-22)[S+30]}&\makecell{move(b,03-22)[S+35]}&\makecell{move(b,03-22)[E]}\\\hline
\hline
$\ttt{time} = 44$&$\ttt{time} = 45$&$\ttt{time} = 100$&$\ttt{time} = 101$\\
\makecell{stop(b,22,II)[S,S]\\stop(b,22,II)[S]}&\makecell{stop(b,22,II)[E]}&\makecell{E[A+100]}&\makecell{depart(b,I,21-02)[S,S]\\depart(b,I,21-02)[S]\\depart(b,II,22-02)[S,S]\\depart(b,II,22-02)[S]\\depart(b,III,23-02)[S,S]\\depart(b,III,23-02)[S]}\\\hline

\hline
$\ttt{time} = 102$&$\ttt{time} = 103$&$\ttt{time} = 104$&$\ttt{time} = 105$\\
\makecell{move(b,21-02)[S,S]\\move(b,22-02)[S,S]\\move(b,23-02)[S,S]\\move(b,21-02)[S]\\move(b,22-02)[S]\\move(b,23-02)[S]}&\makecell{depart(b,I,21-02)[S+$2\epsilon$,S+$2\epsilon$]\\depart(b,II,22-02)[S+$2\epsilon$,S+$2\epsilon$]\\depart(b,III,23-02)[S+$2\epsilon$,S+$2\epsilon$]\\move(b,21-02)[S+5]\\move(b,22-02)[S+5]\\move(b,23-02)[S+5]}&\makecell{depart(b,I,21-02)[E]\\depart(b,II,22-02)[E]\\depart(b,III,23-02)[E]\\move(b,21-02)[S+10]\\move(b,22-02)[S+10]\\move(b,23-02)[S+10]}&\makecell{move(b,21-02)[E]\\move(b,22-02)[S+15]\\move(b,23-02)[S+15]}\\\hline

\hline
$\ttt{time} = 106$&$\ttt{time} = 107$&&\\
\makecell{exit(b,21-02)[S,S]\\exit(b,22-02)[S,S]\\exit(b,23-02)[S,S]\\move(b,22-02)[S+20]\\move(b,23-02)[S+20]\\exit(b,21-02)[S]\\exit(b,22-02)[S]\\exit(b,23-02)[S]}&\makecell{move(b,22-02)[E]\\move(b,23-02)[E]}&&\\\hline
\end{tabular}
}}
\caption{\arpg computed for the \instradi motivating example. $b[\tau^\vdash, \tau^\dashv]$ represents the \IC $c = \tuple{\tau^\vdash, \tau^\dashv, \cond(c)} \in \icond(b)$ and $b[\tau]$ represents the \IE $e = \tuple{\tau, \eff(e)} \in \ieff(b)$. For space reasons, we denoted \tstart, \tend and \talpha with S,E and A respectively. We only show layers of the \arpg where there are snap actions other than the one in $A_\snap^M$ making \ttt{time} flow.}
\label{tab:motivating-arpg}
\end{table}
Let $b$ be an action, we denote with $b[\tau^\vdash, \tau^\dashv]$ the happening associated with the \IC $\tuple{\tau^\vdash, \tau^\dashv, \cond(c)} \in \icond(b)$ and with $b[\tau]$ the happening associated with the \IE $\tuple{\tau, \eff(e)} \in \ieff(b)$. For brevity, we abbreviate $\tstart$ with $\tS$, $\tend$, with $\tE$, $\talpha$ with $\tA$, and $\tomega$ with $\tO$ For example, the action $b = move(\ttt{blue}, \ttt{03-23})$ has the sequence $\AICEs{b}$ equal to
\begin{flalign*}
    \set{b[\tS, \tS];b[\tS]};\set{b[\tS + 5]};\set{b[\tS + 10]};\cdots;\set{b[\tS+25]};\set{b[\tE]}.
\end{flalign*}
The happenings $b[\tS, \tS]$ and $b[\tS]$ have the same condition, being part of the same set $H_1^b$. The snap actions in $A_\snap$  associated with the happenings $b[\tS, \tS]$ and $b[\tS]$ have, in their effects $\exec(b[\tS, \tS]) := \top$ and $\exec(b[\tS]) := \top$, respectively. The snap action in $A_\snap$ associated with $b[\tS + 5]$ has the condition $\set{\exec(b[\tS, \tS]) = \top, \exec(b[\tS]) = \top}$ to impose the order of $\AICEs{b}$.

We denote with $C[\tau^\vdash, \tau^\dashv]$ a plan-\IC $\tuple{\tau^\vdash, \tau^\dashv, \cond(c)} \in C$ and with $E[\tau]$ a plan-\IE $\tuple{\tau, \eff(e)} \in E$. In the motivating example, we can chose the value $M$ to be $101$ since the maximum offset of the \ICEs in $C \cup E$ is $100$. We can thus compile, e.g., the happening $E[\tA + 60]$ into the snap action
\begin{flalign*}
    \tuple{\set{\ttt{time} = 60}, \set{timetable(\ttt{red}), \exec(E[\tA + 60])}}
\end{flalign*}
In $\hat{S}_0 = \Call{Relax}{I}$ the only snap action applicable is $\tuple{\set{\ttt{time} = 0}, \set{\ttt{time} :=1}}$ moving time forward. When $\ttt{time}$'s interval becomes $[0,5]$, then $E[\tA + 5]$ is applicable, i.e., the plan-\IE which causes \ttt{blue} to enter the station. Thus $A_5 = \set{E[\tA + 5], \tuple{\set{\ttt{time} = 5}, \set{\ttt{time} :=6}}}$. In $\hat{S}_6 = res(\hat{S}_5, A_5)$ we have thus $inFront(\ttt{blue}, \ttt{03}) \in \set{\top, \bot}$ and $green(\ttt{03}) \in \set{\top, \bot}$ caused the snap action associated with $E[\tA + 5]$ and $\ttt{time} \in [0,6]$. Table \ref{tab:motivating-arpg} shows how the \arpg computation continues. Each cell represents a layer of the \arpg. We put in the header of each cell the value of \ttt{time}, and we skip the cells where the only snap actions applied are the one making the interval of $\ttt{time}$ increase. The pattern obtained by linearization is, for example, the one where the happenings in each layer of Table \ref{tab:motivating-arpg} is read from top to bottom. Notice that, in the same layer, we order $\ICs$ before \IEs.

\subsection{The Symbolic Pattern Planning Encoding} \label{sec:spp}
We now describe the \spp encoding for Temporal Numeric Planning with \ICEs. Let $\Pi = \tuple{X,A,I,G,C,E}$ be a planning task and let $\prec = h_1;\dots;h_k$ be a pattern of $\Pi$, as described in Sec.\ \ref{sec:pattern-definition}. The \textsl{pattern-encoding} of $\Pi$ using the pattern $\pattern$ is the formula
\begin{gather}
    \Pi^\pattern = \mI(\mX) \land \mS^\pattern(\mX, \mH^\pattern, \mX') \land \mT^\pattern(\mX, \mH^\pattern, \mC^\pattern) \land \mG(\mX'), \label{eq:pattern-encoding}
\end{gather}
where
\begin{enumerate}
    \item $\mX = X$ are\footnote{We employ the calligraphic font to denote sets and formulae belonging to the encoding.} the \textsl{current state variables} and $\mX'$ are the \textsl{next state variables}, which is a copy of the variables in $\mX$, i.e. $\mX' = \set{v' \mid v \in \mX}$,
    \item $\mH^\pattern$ contains a variable with domain in $\natural^{\geq 0}$ for each happening $h_i$ in $\pattern$, denoting how many times $h_i$ is (possibly) repeated (i.e., rolled),
    \item $\mC^\pattern$ contains the \textsl{clock} information, i.e.,
    \begin{enumerate}
        \item a variable $t_i \in \rational^{\geq 0}$ for each $h_i$ in $\pattern$, denoting the time associated to either the start of an \IC or the execution of an \IE,
        \item a variable $t_i^\dashv \in \rational^{\geq 0}$ for each \IC $h_i$ in $\pattern$, denoting the time the \IC ends,
        \item a variable $d_i \in \rational^{\geq 0}$ for each $h_i = b^\vdash$ in $\pattern$, denoting the duration of the durative action $b$, and
        \item a single variable $ms \in \rational^{\geq 0}$ denoting the make-span of the  plan,
    \end{enumerate}
    \item $\mI(\mX)$ is the \textsl{initial state formula} over the variables in $\mX$, i.e., 
    \begin{gather*}
        \bigand_{v: I(v) = \top} v \land \bigand_{w: I(w) = \bot} \neg w \land \bigand_{x,k: I(x) = k} x = k,
    \end{gather*}
    \item $\mS^\pattern(\mX, \mH^\pattern, \mX')$ is the \textsl{causal transition relation}, characterizing the effects of the applications of the \IEs and the respect of the \ICs (discussed in Sec.\ \ref{sec:spp-causal}),
    \item $\mT^\pattern(\mX, \mH^\pattern, \mC^\pattern)$ is the \textsl{temporal transition relation}, characterizing the temporal relations between the durative actions and their happenings, self-overlapping and the $\epsilon$-separation, (discussed in Sec.\ \ref{sec:spp-temporal}),
    \item $\mG(\mX')$ is the \textsl{goal formula} obtained as.
    $$\mG(\mX') = \bigand_{v = \top \in G} v' \land \bigand_{w = \bot \in G} \neg w' \land \bigand_{\psi \unrhd 0 \in G} \psi' \unrhd 0,$$
    where $v'$, $w'$ and $\psi'$ are obtained substituting in $v$, $w$ and $\psi$, respectively, each variable in $X$ with the corresponding one in $\mX'$.
\end{enumerate}

\begin{algorithm}[t]
{\footnotesize
\caption{\spp algorithm. \label{alg:spp}
Input: a temporal planning problem with \ICEs~$\Pi= \tuple{X,A,I,G,C,E}$.
Output: a valid plan for $\Pi$.
}\label{alg:f}
\begin{algorithmic}[1]
\Function{\spp}{$\Pi$} 
    \State $\pattern_g \leftarrow \epsilon$, $P \leftarrow \emptyset$ 
    \State $\pattern_h \leftarrow \textsc{ComputePattern}(I, \Pi)$ \label{alg:f:pattern-h-init}
    \While{\textsc{True}}
        \State $\pattern_f \leftarrow \pattern_g; \pattern_h$ \label{alg:f:concat}
        \State $\mu \leftarrow \textsc{MaxSolve}(
        \Pi^{\pattern_f}) \land \mG(\mX'),G,P)$ \label{alg:f:maxsolve} 
        \If{$|\textsc{SatG}(\mu, G)| = |G|$} \label{alg:f:G}
            \State\!\!\!\!\Return $\textsc{GetPlan}(\mu, \pattern_f)$
        \ElsIf{$|\textsc{SatG}(\mu, G)| > |P|$} \label{alg:f:moreP}
            \State\!\!\!\!$\pattern_g \leftarrow \textsc{Compress}(\pattern_f, \mu)$ 
            \label{alg:f:simplify} 
            \State\!\!\!\!$P \leftarrow \textsc{SatG}(\mu, G)$
            \State\!\!\!\!$s \leftarrow \textsc{GetState}(I,\textsc{GetPlan}(\mu, \pattern_f))$
            \State\!\!\!\!$\pattern_h \leftarrow \textsc{ComputePattern}(s,\Pi)$ \label{alg:f:pattern-h}
        \Else
            \State\!\!\!\!$\pattern_g \leftarrow \pattern_f$ \label{alg:f:else}
        \EndIf
    \EndWhile
\EndFunction
\end{algorithmic}
}
\end{algorithm}

Algorithm \ref{alg:spp} shows the \spp approach used to find a valid model for $\Pi$, taken from \cite{DBLP:conf/kr/CardelliniG25} where it was dubbed $\patty_{dc}$:
\begin{enumerate}
    \item $\Call{ComputePattern}{s, \Pi= \tuple{X,A,I,G,C,E}}$ $(i)$ constructs the \arpg calling the procedure ${\Call{ComputeSnapARPG}{\tuple{X,A,s,G,C,E}}}$ (i.e., where the initial state is replaced with $s$) as shown in Alg. \ref{alg:compute-arpg} of Sec.\ \ref{sec:pattern-computation}, and $(ii)$ performs the linearization and translations from the snap actions of the \arpg.
    \item $\textsc{MaxSolve}(\Pi^{\pattern_f},G,P)$
        calls a \textsc{max-smt} solver returning  an assignment satisfying $\mI(\mX) \land \mS^\pattern(\mX, \mH^{\pattern_f}, \mX') \land \mT^{\pattern_f}(\mX, \mH^{\pattern_f}, \mC^{\pattern_f})$, all the goals in $P$, and a maximal subset of the goals in $G \setminus P$.%
\item $\textsc{GetState}(I,\pi)$ returns the state resulting from the execution of the sequence $\pi$ of actions from $I$.
\item $\textsc{SatG}(\mu, G)$ returns the set of goals in $G$ satisfied by the assignment $\mu$.
\item $\textsc{Compress}(\pattern_f, \mu)$ returns the substring of $\pattern_f = h_1;\dots;h_k$ by selecting the $h_i$ with $\mu(h_i) > 0$ for $i \in [1,k]$. Intuitively, it returns the part of $\pattern_f$ which is "helpful" in reaching $s$.
\item $\Call{GetPlan}{\mu, \pattern}$ returns the plan $\pi$ such that $(i)$ for each happening $h_i$ in $\pattern$ with $\mu(h_i) > 0$ and modeling the start of a durative action $b$, (i.e., $h_i = b^\vdash$), and $(ii)$ for each $r \in \set{1, \dots, \mu(h_i)}$, i.e., each repetition of $b$, we have in $\pi$ the timed durative action
    \begin{flalign*}
        \tuple{\mu(t_i) + (r-1) \times (\mu(d_i) + \epsilon_b), b, \mu(d_i)} ,
    \end{flalign*}
    where $t_i,d_i \in \mC^\pattern$ are the time and duration associated with $h_i$, and $\epsilon_b$ is as described in Eq. \ref{eq:rolling-duration}.
\end{enumerate}
In the \spp procedure, before the search starts, we
 assign $\pattern_g$ to the empty pattern ($\epsilon$), the set $P$ 
(meant to contain the subset of goals that can be satisfied with $\pattern_g$) to the empty set, and we compute an initial pattern $\pattern_h$ from the initial state. Then,
\begin{enumerate}
    \item we set the  pattern $\pattern_f$ used for the search to  $\pattern_g;\pattern_h$ (Line~\ref{alg:f:concat}) and then check whether all the goals in $G$ are satisfied (Line~\ref{alg:f:G}) and, 
    \item 
    if not, we check whether $\pattern_f$ allows satisfying at least one more goal (Line~\ref{alg:f:moreP}), in which case
    we 
    \begin{enumerate}
        \item set $\pattern_g$ to the value of $\textsc{Compress}(\pattern_f, \mu)$
        \item update the set $P$ to the new subset of satisfied goals, 
        \item update the intermediate state,
        \item recompute $\pattern_h$ from the intermediate state $s$ using the \arpg, and 
    \item restart the loop thereby concatenating the newly computed $\pattern_h$ at the next iteration,
    \end{enumerate}
    \item otherwise, (Line~\ref{alg:f:else}) we set $\pattern_g$  to $\pattern_f$, thereby concatenating $\pattern_h$ once more at the next iteration.
\end{enumerate}

As stated in Sec \ref{sec:pattern}, a pattern $\pattern_h = h_1;\dots;h_k$ suggests a causal order between the happenings, i.e., if $h_i$ comes before $h_j$ in $\pattern_h$ (i.e., $i <j$), then $h_i$ \emph{may} cause $h_j$. The conditional is required here because if $h_i$ and $h_j$ are not in mutex, then there is no interference between $h_i$ and $h_j$ and thus $h_j$ could also happen before $h_i$ without breaking the causal relationship suggested by the pattern. If, instead, $h_i$ and $h_j$ are in mutex, the pattern suggests that we should impose that $t_i < t_j$. If the suggestion by the pattern is wrong and every valid plan must contain $h_i$ after $h_j$ (i.e., $|\Call{SatG}{\mu, G}|$ is not increased), then by assigning $\pattern_g \gets \pattern_f$ (which is equal to $\pattern_h$, since at the beginning $\pattern_g \gets \epsilon$ and $\pattern_f \gets \pattern_g;\pattern_h$) in Line \ref{alg:f:else} of Alg. \ref{alg:spp} we obtain, by the concatenation at Line \ref{alg:f:concat}, a pattern $$\pattern_f \gets \underbrace{h_1;\dots;h_i;\dots;h_j;\dots;h_k}_{\pattern_g};\underbrace{h_{k+1}\dots;h_{k+i};\dots;h_{k+j};\dots;h_{2k}}_{\pattern_h}$$ of length $2k$, where $h_{k+i}$ and $h_{k+j}$ are (distinct) copies of $h_i$ and $h_j$, respectively. Thus, we can find a plan having $t_j < t_{k+i}$, i.e., $h_j$ is before $h_{k+i}$ which is a copy of $h_i$.
For this reason, we have split the transition relation in two: 
\begin{enumerate}
    \item $\mS^\pattern(\mX, \mH^\pattern, \mX')$ considers the (state) causal relationship between the happenings of $\pattern$.
    \item $\mT^\pattern(\mX, \mH^\pattern, \mC^\pattern)$ considers the temporal relationship between the happenings of $\pattern$.
\end{enumerate}

\subsubsection{The Causal Transition Relation}  \label{sec:spp-causal}
Let $\Pi = \tuple{X,A,I,G,C,E}$ be a planning task and let $\pattern = h_1;\dots;h_k$ be a pattern. 

\paragraph{Computing the Resulting State} The novelty of patterns lies in the ability to determine the value of $v \in X$ before the (possible multiple) application of each happening in $\pattern$, solely by the initial value of $v$ (determined by $I$) and the application or not of the happenings before $h_i$, i.e., in the subsequence $h_1;\dots;h_{i-1}$ of $\pattern$. We proceed inductively. For each $v \in \mX$, we initialize $\sigma_0(v) = v$ (which we remember is initialized by the initial state formula $\mI(\mX)$). Then for each $i \in [1,k]$ we inductively define $\sigma_i(v)$ as
\begin{enumerate}
    \item $\sigma_i(v) = (\sigma_{i-1}(v) \lor h_i > 0)$ if $v := \top \in \eff(h_i)$, i.e. $v$ is true if it was true before or $h_i$ is executed,
    \item $\sigma_i(v) = (\sigma_{i-1}(v) \land h_i = 0)$ if $v := \bot \in \eff(h_i)$, i.e., $v$ is true if it was true before and $h_i$ is not executed,
    \item $\sigma_i(v) = \sigma_{i-1}(v) + h_i \times \sigma_{i-1}(\psi)$ if $v \pluseq \psi \in \eff(h_i)$ is a linear increment,
    \item $\sigma_i(v) = \ite(h_i > 0, \sigma_{i-1}(\psi), \sigma_{i-1}(v))$ if $v := \psi \in \eff(h_i)$ is not a linear increment,
    \item $\sigma_i(v) = \sigma_{i-1}(v)$ otherwise.
\end{enumerate}
The function $\ite(c, t, e)$, standard in \textsc{smtlib} \cite{BarFT-SMTLIB}, returns $t$ or $e$ depending on whether $c$ is true or not. The value of $\sigma_{i}(\psi)$ is obtained by replacing each $v$ appearing in $\psi$ with $\sigma_i(v)$. Let $\Gamma$ be a set of conditions, we denote with $\sigma_i(\Gamma)$ the formula
\begin{flalign*}
    \sigma_i(\Gamma) = \bigand_{v = \top \in \Gamma} \sigma_i(v) \land \bigand_{w = \bot \in \Gamma} \neg \sigma_i(w) \land \bigand_{\psi \unrhd 0 \in \Gamma} \sigma_i(\psi) \unrhd 0.
\end{flalign*}

\begin{example*}
    Let $h_1 = \tuple{\set{v = \top}, \set{v := \bot, x \pluseq 1}}$, $h_2 = \tuple{\set{v = \bot, x > 4}, \set{x := 5, v := \top}}$, $h_3 = \tuple{\set{x > 5}, \set{x \minuseq 4}}$ and let $\pattern = h_1;h_2;h_3$, then
    \begin{flalign*}
        \sigma_3(v) &= (v \land h_1 = 0) \lor h_2 > 0,\\
        \sigma_3(x) &= \ite(h_2 > 0, x + h_1, 5) - 4h_3.
    \end{flalign*}
    The first row can be read as "after $h_3$, $v$ is true if either $(i)$ $v$ was true in $I$ and $h_1$ is not executed, or $(ii)$ $h_2$ is executed".
\end{example*}

\paragraph{Intermediate Conditions and Rolling} Let $b = \tuple{\icond(b), \ieff(b), [L,U]}$ be a durative action. To correctly encode the rolling of $b$, we have to check if the \ICs in $\icond(b)$ hold at both the first execution of $b$ and, also, after the (possible) following repetitions of $b$. It is clear that the state in which we check the \ICs can depend on the order in which the \IEs in $\ieff(b)$ happens. In the first repetition of $b$, the order of the \IEs is restricted by how the happenings are chosen in the pattern. In the subsequent (possible) repetitions instead, we want to have a formula that represents how the state evolves after each happening by the possible subsequent repetition of $b$. Thus, to construct this formula, we need to know a priori the order of the \IEs of $b$. For this reason, we restrict rolling only to those actions which are both eligible for rolling and well-orderable. The well-orderability condition guarantees that the order given by $\AICEs{b}$ is respected in any valid plan. Let $\psi$ be a linear expression and let $h$ be a happening of $b$. We denote by $\psi[r, b]$ the value of $\psi$ after $r$ repetitions of $b$ up to the happening $h$ in $\AICEs{b}$. Let $\AICEs{b} = H_1;\dots;H_p$ be the order of the happenings of $b$. We recall that there cannot be two \IE with the same offset, i.e. in the same $H_i$ with $i \in [1,p]$, and that $b$ is well orderable. Thus, we can denote with $\AIEs{b}$ the sequence of \IEs of $b$ taken from $\AICEs{b}$ by sequencing all the \IEs in the sets $H_1,\dots,H_p$. Let $\AIEs{b} = h_1;\dots;h_m$ with $m \leq p$. For each $x$ in $\psi$ and for each $h_j$ with $j \in \set{1,\dots,m}$ we define 
\begin{gather*}
    \delta x_j = \begin{cases}
        \psi & \text{if $x \pluseq \psi \in \eff(h_j)$},\\
        0 & \text{otherwise}.
    \end{cases}
\end{gather*}
The expression $\psi[r, b]$ is obtained from $\psi$ by substituting each $x$ in $\psi$ with 
\begin{enumerate}
    \item $\psi'$ if $x := \psi' \in \eff(h_j)$ in some \IE $h_j$, or
    \item otherwise with $x + (r-1) \times \sum_{j=1}^m \delta x_j$.
\end{enumerate}

\begin{example*}
    Let $b$ be a durative action and let $\AIEs{b} = h_1;h_2;h_3;h_4$ with $\eff(h_1) = \set{x \pluseq k_1}$, $\eff(h_2) = \set{x \minuseq k_2}$, $\eff(h_3) = \set{y \pluseq k_3}$, $\eff(h_4) = \set{z := 2}$. Let $\psi$ = $x + y -z$, then 
    $$\psi[r,b] = \underbrace{(x + (r-1)\times(k_1 - k_2))}_{x} + \underbrace{(y + (r-1)\times k_3)}_y - \underbrace{2}_z$$
\end{example*}

\begin{theorem}\label{thm:rolling}
    Let $c$ be a \IC with $\cond(c) = \set{\psi \unrhd 0}$ of a durative action $b$ eligible for rolling and well-orderable. Let $\AICEs{b} = h_1;\dots;h_i;\dots;h_m$, with $h_i = c$ and with all other happenings being \IEs. Let $s$ be the state before applying $b$. The condition $c$ holds for $r \geq 1$ consecutive repetitions of $b$ in $s$ if and only if 
        $$res(s, h_1;\dots;h_{i-1}) \models \psi[r, b] \unrhd 0.$$
\end{theorem}
\begin{proof}
The condition must hold at the first repetition of $b$ (notice that $\psi[1,b] = \psi$) in $res(s, h_1;\dots;h_{i-1})$ i.e., after all the happenings before $h_i$ have been applied. Then, at the next iteration, it must hold at 
$$res(s, \underbrace{h_1;\dots;h_{i-1}}_{r=1};\underbrace{h_i;\dots;h_m;h_1;\dots;h_{i-1}}_{r=2}).$$
Notice that the subsequence corresponding to $r=2$ contains all the happenings of $\AICEs{b}$, thus we can apply all the \IEs in $\AIEs{b}$. The thesis follows from the monotonicity in $r$ of $\psi[r,b]$ (see  \cite{Scala_Ramirez_Haslum_Thiebaux_2016_Rolling}).
\end{proof}  

\paragraph{The Causal Transition Relation} Let $\Pi = \tuple{X,A,I,G,C,E}$ be a planning task and $\pattern = h_1;\dots;h_k$ be a pattern. The causal transition relation, i.e., $\mS^\pattern(\mX, \mH^\pattern, \mX')$, is the conjunction of the formulae in the following sets. 
\begin{enumerate}
    \item $\op{frame}^\pattern(X):$ for each propositional variable $v \in X$ and numeric variable $x \in X$, we have
    \begin{gather*}
        v' \leftrightarrow \sigma_k(v), \qquad x' = \sigma_k(x),
    \end{gather*}
    i.e., the next state variables $v' \in \mX'$ ($x' \in \mX'$) assume the value based on the expressions $\sigma_k(v)$ (resp. $\sigma_k(x)$).
    \item $\op{plan\text{-}intermediate\text{-}causal}^\pattern(C,E)$ composed of the conjunction of the formulae imposing that
    \begin{enumerate}
        \item each plan-\ICs/\IEs in $C$ and $E$ are applied at most once, i.e.,
        \begin{gather*}
            \bigand_{\substack{i \in [0,k] \\ h_i \in C}} (h_i \leq 1) \land \bigand_{\substack{i \in [0,k] \\ h_i \in E}} (h_i \leq 1),
        \end{gather*}
        \item all the plan-\ICs/\IEs in $C$ and $E$ are applied, i.e.,
        \begin{gather*}
            \Big( \sum_{\substack{i \in [0,k] \\ h_i \in C}} h_i \Big) = |C| \land \Big( \sum_{\substack{i \in [0,k] \\ h_i \in E}} h_i \Big) = |E|,
        \end{gather*}
    \end{enumerate}
    \item $\op{action\text{-}intermediate\text{-}causal}^\pattern(A)$ containing, for each durative action $b = \tuple{\icond(b), \ieff(b), [L,U]} \in A$, 
    \begin{enumerate}
        \item and for each occurrence $h_i \in \icond(b) \cup \ieff(b)$ in $\pattern$, the formula imposing that all the $\ICs$ in $\icond(b)$ and all $\IEs$ in $\ieff(b)$ must be applied the same number of times of $h_i$ in at least one of their copies in $\pattern$, i.e.,
        \begin{gather*}
            \bigand_{c \in \icond(b)} \bigor_{\substack{j \in [0,k]  \\ h_j = c}} (h_i = h_j) \land \bigand_{e \in \ieff(b)} \bigor_{\substack{j \in [0,k] \\ h_j = e}} (h_i = h_j),
        \end{gather*}
        intuitively denoting that if $h_i$ is applied, then all other \ICEs of $b$ must be applied in at least one of the copies in $\pattern$,
        \item for each pair $h_p = b^\vdash, h_q = b^\dashv$ in $\pattern$ with the same number of applications, the number of times $b$ is started (via all the copies of $b^\vdash$ in $\prec$ between $h_p$ and $h_q$) should be proportional to the number of times its \ICEs are applied in between $h_p$ and $h_q$, i.e.,
        \begin{gather*}
            h_p = h_q \implies \sum_{\substack{i \in [p,q]\\ h_i = b^\vdash}} h_i \times |\icond(b) \cup \ieff(b)| = \sum_{\substack{i \in [p,q]\\ h_i \in \icond(b)}} h_i + \sum_{\substack{i \in [p,q]\\ h_i \in \ieff(b)}} h_i,
        \end{gather*}
        \item for each $h_p = b^\vdash$ (resp. $h_q = b^\dashv$) applied there is at least one $h_j = b^\dashv$ (resp. $h_j = b^\vdash$) with $j > p$ (resp. $j < q$) and with the same repetitions:
        \begin{flalign*}
            h_p > 0 \implies \bigor_{\substack{j \in (p,k] \\ h_j = b^\dashv}} (h_j = h_p), \qquad h_q > 0 \implies \bigor_{\substack{j \in [1,q) \\ h_j = b^\vdash}} (h_j = h_q),
        \end{flalign*}
    \end{enumerate}
    \item $\op{amo}^\pattern(A)$: for each durative action $b$ not eligible for rolling or not well-orderable, and for each occurrence $h_i = b^\vdash$ in the pattern $\pattern$, the happening can be repeated at-most-once, i.e.,
        \begin{gather*}
            h_i \leq 1,
        \end{gather*}
    \item $\op{conditions}^\pattern(A,C)$: for each $c = \tuple{\tau^\vdash, \tau^\dashv, \cond(c)} \in \icond(\Pi)$ and for each occurrence $h_i = c$ in the pattern $\pattern$ then the condition $\cond(c)$ holds during the first and during the $h_i$-th repetition (Thm. \ref{thm:rolling}), i.e.,
        \begin{flalign*}
            h_i > 0 &\implies \sigma_{i-1}(\cond(c)) \\
            h_i > 1 &\implies ~~~\bigand_{\mathclap{\psi \unrhd 0 \in \cond(c)}}~~~ \sigma_{i-1}(\psi[h_i, b]).
        \end{flalign*}
\end{enumerate}

\subsubsection{The Temporal Transition Relation} \label{sec:spp-temporal}
Let $\Pi = \tuple{X,A,I,G,C,E}$ be a planning task and $\pattern = h_1;\dots;h_k$ be a pattern. The temporal transition relation $\mT^\pattern(\mX, \mH^\pattern, \mC^\pattern)$ has to $(i)$ assign a time to the happenings $\Pi$, $(ii)$ assign a duration to durative actions, $(iii)$ guarantee the $\epsilon$-separation and no-overlapping properties for plan validity, $(iv)$ impose the order of $\pattern$ for happenings in mutex with each other. Let $b$ be a durative action. We have seen how in $\mC^\pattern$ we associate to each happening $h_i = b^\vdash$ a variable $d_i$ denoting the duration of $b$ started by $h_i$. For all other happenings of $b$ appearing in $\pattern$, in the temporal transition relation, we need to retrieve the duration of $b$ in a compact way. For this reason, for each durative action $b \in A$ we construct the following expression inductively: let $\delta_0(b) = 0$, then for each $i \in \set{1,\dots, k}$, we define
\begin{equation}
    \delta_i(b) = \begin{cases}
        \ite(h_i > 0, d_i, \delta_{i-1}(b)) & \text{if $h_i = b^\vdash$}\\
        \delta_{i-1}(b) & \text{otherwise}.
    \end{cases} \label{eq:delta_i(b)}
\end{equation}
\begin{example*}
    Let's suppose we have in $\Pi$ only one action with three happenings and a duration in  $[6,10]$. The three happenings are one at \tstart, one in $\tstart + 4$ and the other at \tend. Let's suppose we have a pattern with two copies of the happenings of $b$, i.e., $\pattern = h_1;h_2;h_3;h_4;h_5;h_6$ with $h_1 = h_4 = b^\vdash$, $h_3 = h_6 = b^\dashv$ and $h_2$ and $h_5$ the happenings in $\tstart + 4$. Then the duration associated to $h_2$ and $h_5$ is $\delta_2(b)$ and $\delta_5(b)$, respectively, as
    \begin{flalign*}
        \delta_2(b) &= \ite(h_1 > 0, d_1, 0),\\
        \delta_5(b) &= \ite(h_4 > 0, d_4, \ite(h_1 > 0, d_1, 0)).
    \end{flalign*}
\end{example*}

Thus, the temporal transition relation $\mT^\pattern(\mX, \mH^\pattern, \mC^\pattern)$ is the conjunction of the formulae in the following sets:
\begin{enumerate}
    \item $\op{make\text{-}span}^\pattern:$ where the make-span $M$ is constrained as
    \begin{gather*}
        \bigand_{\substack{i \in [0,k]\\ h_i = b^\dashv}}ms \geq t_i, \qquad \bigor_{\substack{i \in [0,k]\\ h_i = b^\dashv}} ms = t_i,
    \end{gather*}
    imposing that the make-span $ms$ is greater than all the times in which any durative action ends and equal to at least one time it ends (i.e. the max).
    \item $\op{plan\text{-}intermediate\text{-}temporal}^\pattern(C,E)$: for each plan-\IC $c = \tuple{\tau^\vdash, \tau^\dashv, \cond(c)} \in C$ or plan-\IE $e = \tuple{\tau, \eff(e)} \in E$, let $h_i$ and $h_j$ be the occurrences of $c$ and $e$ in $\pattern$, respectively. We then have the formula imposing the relative time of $c$ and $e$, i.e.,
        \begin{gather*}
            h_i > 0 \implies t_i = \tau[0, ms], \quad h_j > 0 \implies (t_j = \tau^\vdash[0,ms] \land t_j^\dashv = \tau^\dashv[0,ms])
        \end{gather*}
        where $\tau[0,ms]$ and $\tau^\vdash[0,ms]$ are defined by Eq. \ref{eq:tau_a_b}, 
    \item $\op{action\text{-}intermediate\text{-}temporal}^\pattern(A)$ containing, for each durative action $b = \tuple{\icond(b), \ieff(b), [L,U]}$, for each pair $h_p = b^\vdash$ and $h_q = b^\dashv$ in $\pattern$, the formula imposing that if $h_p$ and $h_q$ are applied, then it must exist, for each happening of $b$, one copy in $\pattern$ with the absolute time relative to $t_p$ and $t_q$, i.e.,
        \begin{flalign*}
            h_p = h_q &\implies \bigand_{\substack{c \in \icond(b)\\}} \big( \bigor_{\substack{i \in [p,q] \\ h_i = c}} t_i = \tau^\vdash[t_p, t_q] \big) \land \big( \bigor_{\substack{i \in [p,q] \\ h_i = c}} t_i^\dashv = \tau^\dashv[t_p,t_q] \big), \\
            h_p = h_q &\implies \bigand_{\substack{e \in \ieff(b)}}\bigor_{\substack{i \in [p,q] \\ h_i = e}} t_i = \tau[t_p, t_q],
        \end{flalign*}
    \item $\op{dur}^\pattern(A)$: for each happening $h_i$ in $\pattern$ and each $h_j$ representing an $\IC$ we have that\footnote{In our semantic, we specified that the start of a durative action is in $\rational^{>0}$. To simplify the encoding, we reserve the time in $0$ to happenings which are not applied, so that they do not interfere with the computations.}
    \begin{flalign*}
        h_i = 0 &\iff t_i = 0, \qquad h_j = 0 \iff t_j^\dashv = 0,
    \end{flalign*}
    and for each durative action $b = \tuple{\icond(b), \ieff(b), [L,U]} \in A$ and each $h_i = b^\vdash$ in $\pattern$, we have
    \begin{flalign*}
        h_i = 0 &\implies d_i = 0, \\
        h_i > 0 &\implies L \leq d_i \leq U,\\
        h_i > 0 &\implies \bigor_{\substack{j \in (i,k]\\h_j = b^\dashv}} (t_j = t_i + d_i),
    \end{flalign*}
    imposing that, if $b$ is applied, its duration is in $[L,U]$, and that there is at least one happening modeling the end of $b$ with the time coherent with the duration of $b$.
    \item $\op{\epsilon\text{-}separation\text{-}once}^\pattern$: for each $h_i$ and $h_j$ with $i<j$, in mutex with each other, and not both happenings of some durative action, we have
    \begin{flalign*}
        h_i > 0 \land h_j > 0 \implies t_j \geq t_i^\dashv \qquad & \text{if $h_i$ is an \IC and $h_j$ is an \IE},\\
        h_i > 0 \land h_j > 0 \implies t_j \geq t_i + \epsilon \qquad  & \text{if $h_i$ and $h_j$ are \IEs}.
    \end{flalign*}
    Intuitively, if $h_i$ is an \IC and $h_j$ is an \IE, then, by the semantic, $h_i$ is checked before $h_j$ is applied, so there is no need to $\epsilon$-separate. If $h_i$ is an \IE and $h_j$ is an \IC $c = \tuple{\tau^\vdash, \tau^\dashv, \cond(c)}$, we impose that $h_i$ and the start of $h_j$ are $\epsilon$-separated only if, before $h_i$, the conditions in $\icond(c)$ are not respected\footnote{After $h_i$ they must be respected because $i < j$, due to $\op{conditions}^\pattern(A,C)$ in $\mS^\pattern(\mX, \mH^\pattern, \mX')$}, i.e.,
    \begin{flalign*}
        h_i > 0 \land h_j > 0 \land \neg \sigma_{i-1}(\cond(c)) \implies t_j \geq t_i + \epsilon
    \end{flalign*}

    \item $\op{\epsilon\text{-}separation\text{-}rolling}(A)^\pattern$: for each $h_i$ and $h_j$ with $(i)$ $h_i$ being an happening of a durative action $b$ and $h_j$ being not, $(ii)$ $i<j$ and $(iii)$ $h_i$ and $h_j$ in mutex with each other or different copies of the same happening, we have, if 
    \begin{flalign*}
        h_i > 1 \land h_j > 0 \implies t_j \geq t_i^\dashv + (\delta_i(b) + \epsilon_b) \times (h_i - 1) \quad & \text{if $h_i$ is an \IC},\\
        h_i > 1 \land h_j > 0 \implies t_j \geq t_i + (\delta_i(b) + \epsilon_b) \times (h_i - 1)+ \epsilon \quad  & \text{if $h_i$ and $h_j$ are \IEs},
    \end{flalign*}
    with $\delta_i(b)$ defined in Eq. \ref{eq:delta_i(b)}, i.e., the $h_i$-th repetition of $h_i$ must be $\epsilon$-separated from $h_j$. If $h_i$ is an \IE, $h_j$ is an \IC and $h_j$ is rolled, we impose the $\epsilon$-separation between $h_i$ and the start of $h_j$
    \begin{flalign*}
        h_i > 0 \land h_j > 1 \implies t_j \geq t_i + \epsilon.
    \end{flalign*}
    \item $\op{no\text{-}overlap}^\pattern(A)$: for each action $b \in A$ and for each $h_i = b^\vdash$ and $h_j = b^\vdash$ with $i < j$ we have
    $$h_i > 0 \land h_j > 0 \implies t_j \geq t_i + (d_i + \epsilon_b) \times h_i.$$
\end{enumerate}

\subsection{Correctness and Completeness} \label{sec:correctness-completeness}
Let $\Pi= \tuple{X,A,I,G,C,E}$ be a planning task and let $\pattern = h_1; h_2; \ldots; h_k$, $k \ge 0$ be a pattern. Though the pattern $\pattern$ can correspond to any sequence of plan/action-\ICs/\IEs, it is clear that it is pointless to have in $\prec$ $(i)$ an happening $h_i = b^\vdash$ which is not followed by all the happenings of $b$, $(ii)$ an happening $h_j = b^\dashv$ which is not preceded by all the happenings of $b$, $(iii)$ in between two happenings $h_p = b^\vdash$ and $h_q = b^\dashv$, not have copies of all the happenings in $\icond(b) \cup \ieff(b)$, $(iv)$ if the planning task is well-orderable, an order between the happenings in $\prec$ that doesn't follow the one specified by $\AICEs{b}$ for each durative action $b$, $(v)$  an happening $h_j$ corresponding to an $\IE$ of an action $b$ before an happening $h_i$ corresponding to an $\IC$ of an action $b$ if they have the same relative time. On the other hand, it makes sense to consider patterns with non-consecutive occurrences of the same happening. 

\begin{example*} 
    Suppose there are two durative actions: $b_1$ with happenings $h_1=b_1^\vdash$, $h_2=b_1^\dashv$ and $b_2$ with happenings $h_3=b_2^\vdash$ and $h_4=b_2^\dashv$, with $h_2$ in mutex with $h_3$. It could make sense to consider the pattern
    \begin{equation}
        \pattern = h_1;h_2;h_3;h_2;h_4;h_2, \label{eq:non-simple}
    \end{equation}
    to suggest that $b_1$ starts before $b_2$ but we are not sure whether $b_1$ will end before/during/after $b_2$. This pattern allows covering all three cases without extending the pattern.
\end{example*}

No matter how $\pattern$ is defined,
 the {\sl $\pattern$-encoding $\Pi^\pattern$ of $\Pi$ (with bound $1$)} is correct, where $\Pi^\pattern$ is defined in Eq. \ref{eq:pattern-encoding}. We denote with $\indpi$ the sequence of happenings obtained by sequencing each \ICs $c$ of $\Pi$ such that $\tuple{t^\vdash, t^\dashv, \cond(c)} \in \icond(\pi)$ (Eq. \ref{eq:icond}) and each \IE $e$ of $\Pi$ such that $\tuple{t, \eff(e)} \in \ieff(\pi)$ (Eq. \ref{eq:ieff}) according to their execution times (i.e., $t^\vdash$ and $t$). Completeness is guaranteed once we ensure that $\indpi$ is a subsequence of the pattern used in the encoding. This can be achieved by starting with a complete pattern, and then repeatedly chaining it until $\Pi^\pattern$ becomes satisfiable (Alg. \ref{alg:spp}).
We saw in Sec. \ref{sec:pattern-definition}, the definition of a complete pattern. We now define $\pattern^n$ to be the sequence of actions obtained concatenating $\pattern$ for $n \ge 1$ times. Finally,
$\Pi^\pattern_n$ is the {\sl pattern $\pattern$-encoding of $\Pi$ with bound $n$}, obtained from Eq. \ref{eq:pattern-encoding} by considering $\pattern^n$ as the pattern $\pattern$.

\begin{theorem} \label{th:compl}
Let $\Pi$ be a planning task. Let $\pattern$ be a pattern. 
Any model of $\Pi^\pattern$ corresponds to a valid temporal plan of $\Pi$ (correctness). If $\Pi$ admits a valid temporal plan and $\pattern$ is complete, then for some $n\ge 0$, $\Pi^\pattern_n$ is satisfiable (completeness).
\end{theorem}

\begin{proof}
    Correctness: Let $\mu$ be a model of $\Pi^\pattern$ and $\pi$ its associated plan.  The $\epsilon$-separation axioms ensure that the relative order between mutex actions in $\pi$ and in $\pattern$ is the same. The causal transition relation ensures that  executing sequentially the actions in $\pi$ starting from $I$  leads to a goal state. The axioms in the pattern time encoding are a logical formulation of the corresponding properties for the validity of $\pi$. Completeness: Let $\pi$ be a valid temporal plan. The formula $\Pi^{\indpi}$, i.e., where $\indpi$ is used as pattern, is satisfied by the model $\mu$ whose associated plan is $\pi$. Let $n = |\indpi|$ be the number of happenings in $\indpi$. For any complete pattern $\pattern$, $\indpi$ is a subsequence of $\pattern^n$ and  $\Pi^\pattern_{n}$ can be satisfied by extending $\mu$ to assign $0$ to all the action variables not in $\indpi$. 
\end{proof}

\begin{example*}
    Suppose that the only valid plan $\pi$ is such that 
    $\indpi = h_1;h_2;h_3$. Suppose that we employ a pattern $\pattern = h_3;h_2;h_1$. Then, $\indpi$ is a subsequence of $\pattern^3$, i.e.,
    $$\pattern^3 = h_3;h_2;h_1;h_3;h_2;h_1;h_3;h_2;h_1,$$
    and $\Pi^\pattern_3$ is satisfied.
\end{example*}

\begin{theorem}[\cite{DBLP:conf/kr/CardelliniG25}]
    Let $\Pi$ be a planning task. The \spp algorithm is {\sl correct} (any returned plan is valid) and {\sl complete} (if a valid plan exists, $\spp(\Pi)$ will return one).\footnote{Here, we follow the standard definition of completeness used in planning. As for all other numeric planning procedures, if no valid plan exists, the procedure doesn't terminate.}
    \begin{proof}
        (Correctness) It follows from the correctness of $\Pi^\pattern$. (Completeness) Let $\pi$ be a valid temporal plan. Let $n = |\indpi|$ be the number of happenings in $\indpi$. Let $\pattern_h$ be the complete pattern computed by $\Call{ComputePattern}{I, \Pi}$. If no goal is ever satisfied, Lines \ref{alg:f:else} and \ref{alg:f:concat}, at the $n$ iteration, construct the pattern $\pattern_h^n$. The rest of the completeness follows from Thm. \ref{th:compl}.
    \end{proof}
\end{theorem}

Notice that when two happenings $h$ and $h'$ are not in mutex and they are not denoting the starting/ending of the same action, the pattern does not lead to an ordering on their execution times. For this reason, we may find a valid plan $\pi$ for $\Pi$ even before $\pattern^n$ becomes a supersequence of $\indpi$.

\section{Experimental Analysis} \label{sec:experimental}
        
We developed our Temporal Numeric Planning with \ICEs encoding in the \patty solver\footnote{\url{https://pattyplan.com}} which implements the \spp algorithm showed in Alg. \ref{alg:f} and in \cite{DBLP:conf/kr/CardelliniG25}. We considered 8 domains coming from the literature, at which we added our motivating example, the \instradi domain. Before presenting and commenting the analysis, we first describe the domains and the planners.

\paragraph{Domains}
We run our experimental analysis on the following domains. After the name of the domain, we write whether the domain contains \ICEs or not (denoted with \ICEs or \textsc{t}). Domains without \ICEs are written in \pddl{2.1} \cite{Fox_Long_2003}, while domains with \ICEs are expressed in \anml \cite{smith2008anml} since \pddl{2.1} doesn't have a way to express \ICEs.\footnote{Download the domains at \url{https://matteocardellini.it/aij/domains_aij.zip}.}
\begin{description}
    \item \textsc{Cushing (t)} is the only domain with \textsl{required concurrency} \cite{DBLP:conf/ijcai/CushingKMW07} in the 2018 International Planning Competition ({\ipc}) \cite{ipc2018}, the last with a temporal track. The domain implements into a planning task Figure 1 of the \citet{DBLP:conf/ijcai/CushingKMW07} paper, having different actions that represent different kinds of \textsl{temporal expressiveness} \cite{DBLP:conf/ijcai/CushingKMW07}.
    \item \textsc{Pour (t)} implements the motivating example in \cite{Cardellini_Giunchiglia_2025}: we have $p$ bottles on the left with some integer litres of liquid and $q$ empty bottles on the right. We can pour from one bottle on the left to a bottle on the right, pouring one litre per second, only when bottles are uncapped. A durative action uncaps a bottle only for $5$s. All the pouring must happen during this durative action. In the domain, one could exploit rolling, since the pouring action is eligible for rolling.
    \item \textsc{Shake (t)} \cite{Cardellini_Giunchiglia_2025}: a bottle of soda can be shaken with duration $4s$ and can be capped for $5s$. The shake action can start only if the bottle is capped and must end when the bottle is uncapped, and its effect is to lose all the contained litres of the bottle. In the initial state, all the bottles are uncapped and have some liquid inside, the goal state is to have them all empty. To achieve the goal, the shake action must always start during the capping action and must always finish after its end.
    \item \textsc{Pack (t)} \cite{Cardellini_Giunchiglia_2025}: the bottles have to be packed in pairs. There is a pack action of a bottle that, at its start, increases by one a counter that counts the number of bottles loaded on the packing platform. The action can start if there are less than two bottles in the packing platform and can end if there are exactly two bottles, completing the packing process. An instantaneous action resets the counter. Due to this construction, two pack actions must concur to pack two bottles together, and one pack action has to start before the other pack action ends.
    \item \textsc{Match (t)} \cite{halsey2004crikey}: the only domain with required concurrency appearing in the temporal track of the 2014 \ipc \cite{ipc2014}. In the domain, we are required to light a match, through a durative action that lasts $70s$ while concurrently fixing several fuses.
    \item \textsc{Oversub (t)} \cite{Panjkovic_Micheli_2023}: a synthetic problem in which it is encoded a oversubscription planning problem.
    \item \instradi \textsc{(\ICEs)}: the motivating example of this paper. We model instances from $1$ to $20$ trains that run on the station depicted in Fig. \ref{fig:instradi}. Each train arrive either at entry point \ttt{01} or \ttt{03}, has to stop at a platform, and needs to exit either from exit point \ttt{04} or \ttt{02}, respectively.
    \item \textsc{Majsp (\ICEs)} \cite{valentini_temporal_2020}: A job-shop scheduling problem in which a fleet of moving agents transport items and products between operating machines.
    \item \textsc{Painter (\ICEs)} \cite{valentini_temporal_2020}: a worker has to apply several coats of paint on a set of items guaranteeing a minimum and a maximum time between two subsequent coats on the same item.
\end{description}

\paragraph{Planners} 
The analysis compares our system \patty implemented by modifying the planner in \cite{DBLP:conf/aaai/CardelliniGM24}, written in Python, and employing the {\smt}-solver \ttt{Z3} v4.12.2 \cite{de2008z3}; the symbolic planners \anmlsmt (which corresponds to $\textsc{anml}^{\textsc{omt}}_{\textsc{inc}}(\textsc{omsat})$ in \cite{Panjkovic_Micheli_2023}) and \textsc{itsat} \cite{RankoohG15}; and the search-based planners \tamer \cite{valentini_temporal_2020}, \optic \cite{optic}, \textsc{lpg} \cite{LPG} and \textsc{TemporalFastDownward} (\textsc{tfd}) \cite{tfd}. \anmlsmt and \optic have been set to return the first valid plan they find. \patty, \tamer and \anmlsmt are the only planners which can parse \anml and deal with \ICEs. For planners which cannot deal with \ICEs we compiled the domains \instradi \textsc{(\ICEs)}, \textsc{Majsp (\ICEs)} and \textsc{Painter (\ICEs)} into \pddl{2.1} using the compilation by \citet{gigante_expressive_2022}. For all solvers, we set $\epsilon = 0.001$.

\begin{table}[t]
    \centering
    \resizebox{\textwidth}{!}{\Huge{
    \begin{tabular}{|l|c||ccccccc||ccccccc||ccc||}
\hline
 & & \multicolumn{7}{c||}{Solved}&\multicolumn{7}{c||}{Time (s)}&\multicolumn{3}{c||}{Bound}\\
Domain & \# & $\textsc{patty}$&$\textsc{tamer}$&\textsc{anmlsmt}&\textsc{Optic}&\textsc{itsat}&\textsc{lpg}&\textsc{tfd}&$\textsc{patty}$&$\textsc{tamer}$&\textsc{anmlsmt}&\textsc{Optic}&\textsc{itsat}&\textsc{lpg}&\textsc{tfd}&$\textsc{patty}$&\textsc{anmlsmt}&\textsc{itsat}\\
\hline
\textsc{Cushing (t)}&10&\textbf{10}&2&4&\textbf{10}&-&-&1&13.47&240.12&195.81&\textbf{3.12}&-&-&270.02&\textbf{3.0}&9.0&-\\
\textsc{Pour (t)}&20&\textbf{15}&6&1&-&-&-&-&\textbf{123.54}&222.69&286.24&-&-&-&-&\textbf{2.0}&14.0&-\\
\textsc{Shake (t)}&20&\textbf{20}&-&10&-&-&-&-&\textbf{2.26}&-&153.53&-&-&-&-&\textbf{2.0}&8.5&-\\
\textsc{Pack (t)}&20&7&\textbf{20}&4&-&-&-&-&205.15&\textbf{23.08}&242.53&-&-&-&-&\textbf{2.5}&11.0&-\\
\textsc{Match (t)}&40&\textbf{40}&\textbf{40}&\textbf{40}&\textbf{40}&\textbf{40}&-&-&3.10&\textbf{0.01}&0.36&\textbf{0.01}&0.68&-&-&\textbf{3.2}&9.0&4.0\\
\textsc{Oversub (t)}&20&\textbf{20}&\textbf{20}&\textbf{20}&\textbf{20}&-&\textbf{20}&-&4.36&7.15&0.04&\textbf{0.01}&-&0.08&-&\textbf{1.0}&3.0&-\\
\textsc{InSTraDi} (\textsc{ice}s)&20&\textbf{20}&4&2&-&-&-&-&\textbf{67.58}&247.41&279.54&-&-&-&-&\textbf{1.0}&20.5&-\\
\textsc{Majsp} (\textsc{ice}s)&20&16&\textbf{19}&10&-&-&-&-&91.97&\textbf{30.33}&158.03&-&-&-&-&\textbf{4.5}&16.0&-\\
\textsc{Painter} (\textsc{ice}s)&20&\textbf{20}&\textbf{20}&7&2&-&-&-&17.55&\textbf{0.05}&210.15&270.03&-&-&-&\textbf{1.5}&13.5&-
\\\hline
\textit{Best}&190&\textbf{168}&\textbf{131}&\textbf{98}&\textbf{72}&\textbf{40}&\textbf{20}&\textbf{1}&\textbf{44}&\textbf{79}&\textbf{9}&\textbf{70}&\textbf{0}&\textbf{0}&\textbf{0}&\textbf{168}&\textbf{0}&\textbf{16}\\\hline

        \end{tabular}}}
    \caption{Comparative analysis. A “-” indicates that no result was obtained in our 5min time limit. The best results are in bold.}
    \label{tab:results}
\end{table}
\paragraph{Analysis of the results} Table \ref{tab:results} shows the name of each domain, the number of instances for each domain, the number of instances solved (subtable "solved") for each planner, the average time to solve the instances, having the time equal to timeout when the instance is not solved (subtable "Time (s)"), and for planners based on $\pas$, the average number of steps (or bound) to solve the instances commonly solved by all planners that could solve at least one instance (subtable "Bound"). The best results are in bold.
We also include a final line, labelled \textit{Best}, reporting on how many of the considered 190 problems each planner obtained the best result. Each planner had a time limit of $5$ minutes on an Intel Xeon Platinum $8000$ $3.1$GHz with 8 GB of RAM. 

Tab. \ref{tab:results} shows that, out of the $190$ instances, our solver \patty is the one solving the most ($168$), followed by \tamer ($131$), \anmlsmt ($98$), \optic ($72$), \textsc{itsat} ($40$), \textsc{lpg} ($20$) and $\textsc{tfd}$ ($1$). Concerning time, several planners, like \tamer, \optic, \itsat and \lpg can almost instantaneously solve all instances of \textsc{Match (t)} and \textsc{Oversub (t)}. \optic is very fast on the \textsc{Cushing (t)} domain and \tamer is very fast on the \textsc{Painter (\ICEs)} domain. Concerning the bound, we see how \patty has always a lower bound than all other \smt solvers. 
Interestingly, in the \instradi \textsc{(\ICEs)} and \textsc{Oversub (t)} domains, \patty solves all instances with bound $n=1$. Indeed, for those problems, the pattern computation can correctly produce the causal ordering between the actions in the plan, and what is left to do is to correctly choose which actions to apply and assign a timing to each durative action.

\begin{figure}[t]
    \centering
    \includegraphics[width=1\linewidth]{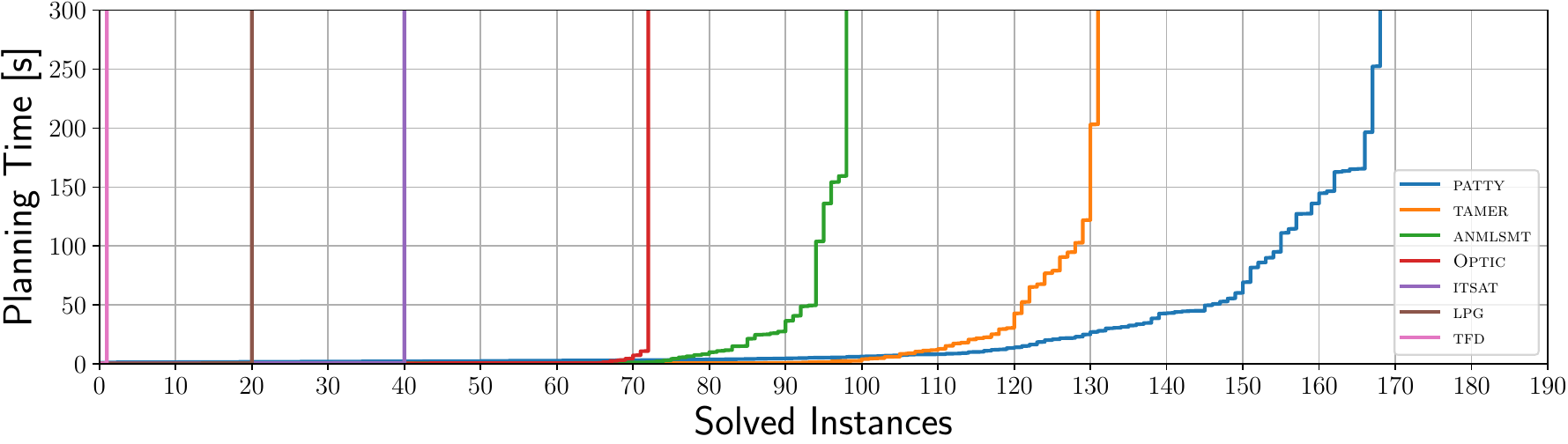}
    \caption{Number of problems solved ($x$-axis) in a given time ($y$-axis), by all the presented systems. In the legend, the planners are listed in reverse order of when their curve intersects the timeout line.
    }
    \label{fig:cactus}
\end{figure}

The cactus plot in Figure~\ref{fig:cactus} summarizes the performance of all the systems we presented. The graph plots how many problems can be solved in a given time. As it can be seen, \tamer can solve the first $106$ problems a little bit faster than \patty, but then, afterwards, \patty is consistently faster. We believe this is because $(i)$ \patty is implemented in Python, while \tamer is implemented in \textsc{c$^{++}$}, a notably faster language, $(ii)$ while heuristic search methods (like \tamer) can start immediately after parsing and grounding the problem, in our approach we have to first compute the pattern with Alg. \ref{alg:compute-arpg}, that, albeit being linear in the number of actions, still requires some time, and $(iii)$ \patty compiles the problem into an \smt formula, encoded in \textsc{smt-lib} \cite{BarFT-SMTLIB} that is then passed to the external solver \ttt{Z3}, which take some additional time to initialize. Despite this, \patty can, in the long rung, solve more instances than \tamer.

The results in Tab. \ref{tab:results} and Fig. \ref{fig:cactus} show the competitiveness of our approach. It is clear that some planners are better suited for some domains over other planners, but our planner \patty can solve the broadest number of instances across all domains. Moreover, \patty has competitive performances against search-based solvers, showing that \pas approaches can still be competitive against them.

\section{Conclusions and Future Works}
In this paper, we extended the work of \cite{DBLP:conf/aaai/CardelliniGM24} from Numeric Planning to Temporal Numeric Planning with \ICEs. In this fragment, we showed a novel approach to compute patterns, showing how we can reuse the \arpg construction by \citet{Scala_Haslum_Thiebaux_Ramirez_2016_AIBR}  for this fragment as well, and we presented a novel encoding based on the \spp approach. We performed an experimental analysis on our motivating example, the \instradi domain, and
several domains from the literature, showing how the \spp approach can be very competitive against other solvers in the field.

This paper showed that the \spp approach is well suited for various fragments of planning. 
In future, we plan to investigate how to encapsulate several search strategies, already present in the literature of heuristic search, inside the \spp approach, thus improving on Alg. \ref{alg:f}. 
For example, in Alg. \ref{alg:f}, the pattern $\pattern_h$ is recomputed at each intermediate state $s$ which satisfies a new subgoal, but the formula $\Pi^\pattern$ always constraints the first state to be the initial one, and thus we need to concatenate $\pattern_h$ with $\pattern_g$, i.e., the pattern of the plan to go from the initial state to $s$. 
One could ask if it would be beneficial to discard $\pattern_g$ and, in a greedy fashion, restart searching from $s$ with only $\pattern_h$. 
Moreover, the computed $\pattern_h$ is always a complete pattern, containing all the \ICEs of each durative action together with the plan-\ICEs in $C \cup E$. 
One could ask if it would be beneficial, and how, to discard some happenings in the pattern, making it incomplete, and only use some most promising one to reach a new subgoal. 
These approaches would greatly improve the speed of the planner, since the number of variables in the encoding is proportional to the length of the pattern, but would nonetheless result in a loss of completeness.
One could then investigate how to maintain completeness \emph{eventually}: i.e., guaranteeing that, if the search from $s$ with only $\pattern_h$ is unsuccessfully, the procedure reverts to the safe approach of Alg. \ref{alg:f} from the initial state and with $\pattern_g;\pattern_h$, or that if the use of an incomplete pattern is unsuccessful, we can increasingly add happenings to it, until it becomes a complete pattern.

Moreover, both in Numeric Planning in \cite{DBLP:conf/aaai/CardelliniGM24} and here in Tab. \ref{tab:results} we notice that there are several domains where all instances can be solved with bound $n=1$. This indicates that the problem of finding the causal order between actions (i.e., the pattern) is an easy task, and the \smt-solver's only job is to find which actions to execute and, in the case of temporal planning, the timings of the durative actions. One could thus investigate in which kinds of problems is always possible to solve the problem with $n=1$. Moreover, in our \spp approach, the pattern $\pattern_h$ computed at each intermediate state is always \textsl{simple} and complete, i.e., it contains exactly one copy of each happening. One could thus investigate on which domains it would be possible to have bound $n=1$ by lifting the simplicity of the pattern and computing a non-simple pattern, like the one in Eq. \ref{eq:non-simple}.

\section*{Declaration of Competing Interest }

\noindent The authors declare that they have no known competing financial interests or personal relationships that could have appeared to influence the work reported in this paper.

\section*{Data availability}

\noindent A link to the source code is included in the article. The domains are attached to the submission

\section*{Acknowledgements}

\noindent Matteo Cardellini and Enrico Giunchiglia acknowledge the financial support from PNRR MUR Project PE0000013 
 FAIR “Future Artificial Intelligence Research”, funded by the European Union – NextGenerationEU, CUP J33C24000420007, and from Project PE00000014 “SEcurity and RIghts in the CyberSpace”, CUP D33C22001300002.


 \bibliographystyle{elsarticle-harv} 
 \bibliography{cas-refs}
\end{document}